\documentclass[conference]{IEEEtran}

\usepackage{cite}
\usepackage{amsmath,amssymb,amsfonts}
\usepackage{algorithmic}
\usepackage{graphicx}
\usepackage{textcomp}
\usepackage{xcolor}
\usepackage[hidelinks]{hyperref}
\usepackage[numbers]{natbib}

\def\BibTeX{{\rm B\kern-.05em{\sc i\kern-.025em b}\kern-.08em
    T\kern-.1667em\lower.7ex\hbox{E}\kern-.125emX}}

\usepackage{amsfonts,amsmath,amssymb,amsthm}
\usepackage{booktabs} %
\usepackage[normalem]{ulem} %
\usepackage{xspace}
\usepackage{bbm}

\newif\ifnotes

\notesfalse

\newif\ifsubmission

\submissionfalse

\DeclareMathOperator*{\argmax}{arg\,max}

\newtheorem{theorem}{Theorem}
\newtheorem{claim}{Claim}

\newtheorem{definition}{Definition}
\newtheorem{conjecture}{Conjecture}

\newcommand{\expect}{\mathop{\mathbb{E}}}

\usepackage{xcolor}
\ifnotes
    \newcommand{\mynote}[3]{\marginpar{{\tiny \parbox{0.8in}{\color{#2} \sf {#1}: {#3}}}}}
    \newcommand{\inlinenote}[3]{{\bfseries \color{#2} {#1}: #3}}
    
\else
    \newcommand{\mynote}[3]{}
    \newcommand{\inlinenote}[3]{}
    
\fi

\newcommand{\hide}[1]{\mynote{Some text hidden here.}{red}{}}

\newcommand{\ot}[1]{{\color{lightgray}{\sout{#1}}}}

\renewcommand{\ot}[1]{}%

\theoremstyle{definition}

\newtheorem*{theorem*}{Theorem}

\theoremstyle{remark}

\newcommand{\X}{\mathcal{X}}
\newcommand{\Y}{\mathcal{Y}}

\newcommand{\reals}{\mathbb{R}}
\newcommand{\req}{\mathtt{Req}}

\newcommand{\opt}{\mathtt{opt}}

\newcommand{\alg}{\mathcal{A}}
\newcommand{\RA}{\mathcal{R}_\alg}

\newcommand{\rand}{\mathtt{Rand}}

\newcommand{\retrain}{\mathtt{Retrain}}

\newcommand{\D}{\mathcal{D}}
\newcommand{\F}{\mathcal{F}}

\newcommand{\err}{\mathtt{err}}
\newcommand{\util}{\mathtt{util}}

\newcommand{\poisson}{\text{Poisson}}

\newcommand{\shap}{\mathtt{Shap}}

\newcommand{\fn}{h}

\submissiontrue
\begin{document}

\title{Modeling Deletion Requests in Machine Unlearning}

\author{\IEEEauthorblockN{Christian Cianfarani}
\IEEEauthorblockA{\textit{Data Science Institute} \\
\textit{Columbia University}\\
New York, USA \\
cc5672@columbia.edu}
\and
\IEEEauthorblockN{Aloni Cohen}
\IEEEauthorblockA{\textit{Department of Computer Science} \\
\textit{University of Chicago}\\
Chicago, USA \\
aloni@uchicago.edu}
}
\maketitle

\begin{abstract}
    Machine unlearning is seen as a promising approach to enable users to exercise the ``right to erasure'' in the context of AI models. We ask how users might influence the behavior of models when exercising this right. We define two types of behaviors that users might adopt when requesting the deletion of their data: \emph{adaptivity} and \emph{collectivity}. Drawing connections between the goals of users in this context and results in stochastic optimization, we demonstrate theoretical gaps between the potential effects of groups of users who do and do not display these behaviors. We then show how techniques from data valuation might be used to design deletion requesters that can significantly alter model behavior in realistic settings. In experiments on computer vision tasks, we demonstrate the differential effects of different models of user behavior and attempt to isolate the impacts of adaptivity and collectivity.
\end{abstract}
\begin{IEEEkeywords}
Data Deletion, Machine Unlearning, Privacy.
\end{IEEEkeywords}

\section{Introduction}
The ``right to be erasure'' gives individuals the ability to request their personal data be deleted from third-party databases. Versions of this right have been conferred by the EU’s General Data Protection Regulation (GDPR) and the California Consumer Privacy Act (CCPA). A growing body of academic work has studied how this right might apply to personal data encoded in the weights of large machine learning models. The methods proposed in these works are broadly encompassed by the term \emph{machine unlearning}. While these methods are not yet suitable for large-scale deployments, they demonstrate that, in principle, personal data can be ``erased’’ from ML models. 

Most research on machine unlearning is about how to unlearn quickly and effectively while controlling the impact on model performance. The proposed methods are often evaluated by unlearning \emph{uniformly random} samples from the training set. 

Suppose that individuals may one day be able to request their data be deleted from the weights of machine learning models after deployment. 
One should not expect deletion requests to be uniformly random. They will depend on the requesters' desires, characteristics, social circles, and the data itself.

We wish to understand how the requesters' behaviors when issuing deletion requests can influence model performance after unlearning. 
This will depend on characteristics of the requesters.

For instance, users may be able to see the outcomes of previous model updates before making the decision to delete their data. We call this property \emph{adaptivity}. We illustrate this concept in the following examples:

\vspace{4pt}
    \noindent\textbf{Example 1 (Adaptive Requests).} An online querying interface for a new LLM is made public. The model deployer promises to unlearn whenever it receives a deletion request, updating the model weights immediately. 
    Alice interacts with the model soon after it is released and decides against requesting deletion. However, after a few weeks during which the weights of the model were updated several times, Alice learns that the model is now regurgitating her personal information and she decides to request the removal of her data. 

\vspace{4pt}
    \noindent\textbf{Example 2 (Nonadaptive Requests).} A company is developing an open-weight model but wants to give the public the opportunity to request their data be unlearned before the weights are published. An online querying interface for the initial model is made public and individuals are allowed to request their data be deleted. After a period of time, all requests the company received are unlearned together and the final weights of the model are published. During the querying phase, Bob observes that the model does not regurgitate his personal information and does not request deletion. After the deletion requests are processed and the final model is published, however, the model does regurgitate Bob's personal information.

\vspace{4pt}
Alice has recourse because she can act adaptively: she can react to the effects of other deletion requests. Bob is nonadaptive: he's stuck with the model once it's published.

Users may also coordinate with each other when issuing deletion requests. We refer to this property as \emph{collectivity}, and the absence of this property as \emph{individuality}.
We illustrate collective and individual requests with the following examples.

\vspace{4pt}
    \noindent\textbf{Example 3 (Individual Requests).} A model is trained on data collected from participants in a medical study. Participants know that their data was used in training, but for privacy reasons do not know each other's identities and are unable to communicate with one another. When making the decision to request the deletion of their data, participants therefore have no knowledge of the training dataset beyond what is revealed by the model itself and their own data.

\vspace{4pt}
    \noindent\textbf{Example 4 (Collective Requests).} A company releases a language model trained on romance novels authored by members of the Romance Author Guild. Some of the individual authors are concerned about the presence of their work in the training data, but defer to the guild on whether or not to request their data be deleted. The guild holds a vote among its members and decides that they will all request the deletion of their data. 

\vspace{4pt}
In Example 3, participants in the study are incapable of exchanging information with one another, so each participant's decision to remove their data is independent of the training dataset conditioned on the initial state of the model. In Example 4, individuals only make the decision to delete after coordinating with others whose data was trained on.

\subsection{Contributions}

This paper asks: How do different deletion requester behaviors differently impact downstream model performance after unlearning?

\subsubsection{Characterizing Deletion Requesters}
We present two qualities that can characterize the behavior of deletion requests: \emph{adaptivity}, or the ability of a request to respond to previous requests and model updates, and \emph{collectivity}, or the ability to act on knowledge of the full training dataset. 
In order to study the effects of different classes of requesters, we consider the adversarially \emph{optimal requester} that degrades model performance as much as possible. 

\subsubsection{Understanding Gaps in Performance}
We define a simple learning setting within which we study the gaps between optimal requesters of different types.
In that setting, we reduce optimal unlearning requesters to instances of the (correlated stochastic) knapsack problem. We show separations between optimal collective and individual requesters, as well as between adaptive and nonadaptive requesters. These results imply that (1) personal information shared between individuals and (2) information leaked by the model itself could alter the impacts of deletion requests on model behavior.

\subsubsection{Designing Strong Requesters}
Since solving the knapsack problem is intractable in the worst case, we show how one can design requesters that outperform random requests. Drawing parallels between unlearning and data valuation, we design and implement requesters based on the Shapley value. We show that although these Shapley requesters are provably suboptimal in certain settings, they clearly outperform random requests in image classification experiments. We also empirically study how adaptivity and collectivity affect the performance of Shapley-based requesters.

\section{Related Work}
Machine Unlearning was first defined in its modern form by \citet{cao2015towards}, who propose modifications to existing learning systems to allow for unlearning.\ Follow up works provide more formal definitions and propose unlearning methods for more general learning algorithms~\cite{ginart2019making,bourtoule2021machine,guo2019certified}. A limitations of these definitions is that they fail to provide privacy guarantees when requests are allowed to be dependent on previous deletions and model updates. \citet{gupta2021adaptive} were the first to formalize this setting, which they called \emph{adaptive machine unlearning}. They show that the unlearning guarantees of the SISA unlearning framework~\cite{bourtoule2021machine} fail to hold against adaptive requesters, and give a method for transforming nonadaptive guarantees to adaptive guarantees using differential privacy.
Follow up works \citep{chourasia2023forget,cohen2023control, cohen2026protecting} further refine and strengthen these results, although non-adaptive definitions remain popular in the literature.

\subsection{Evaluations of Unlearning}
A simple framework for evaluating an unlearning method is as follows.
First, a model is trained on a training dataset. That dataset is then split into a forget set (the set of points to be unlearned) and a retain set (the set of points to be retained). The unlearning method is applied to the original model to remove the influence of the forget set, and a new model is trained from scratch on the retain set. These models are then used to measure the performance of the unlearning method.

Within the machine unlearning literature, there is little standardization in how performance is measured.\ \citet{li2025machine} provide a taxonomy of existing approaches. The metrics used to measure unlearning performance are divided into verification metrics, which can be used to test whether data in the forget set has been fully removed, and evaluation metrics, which measure how well an unlearning method approximates retraining from scratch.

Another aspect of machine unlearning evaluations is the selection method for the forget set.
A common approach is random forgetting: the forget set is sampled uniformly from the training set. 
Another approach is class unlearning, where the entirety of a given class is deleted from the training set~\cite{jia2023model}. 
Other papers seek to unlearn a collection of points that share certain qualities. For example, unlearning all of the points that have been targeted by a data poisoning attack~\cite{goel2024corrective}, or the set of points that encode a certain behavior in a generative model~\cite{ilharco2022editing}. \citet{carlini2022privacy} remove points that are most susceptible to membership inference attacks and find that new points become more susceptible to these attacks afterwards, highlighting unintended privacy implications of unlearning.

A distinct line of work has studied how the choice of forget set can impact the \emph{fairness} of a model. Several papers~\cite{oesterling2024fair,zhang2024forgotten,liu2025enabling} investigate fairness impacts when deletion requests are biased towards minority or majority groups.\ \citet{surve2025explaining} study the use of machine unlearning as an explainability technique for understanding fairness violations. They propose optimization techniques to identify the subset of the training data that contributes the most to the unfairness of the model.

\subsection{Data Valuation}
Our approach to evaluating the impact of deletion requests is inherently tied to the topic of data valuation. A sizable body of work supports the use of the Shapley value to assign importance scores to individual data points, especially in machine learning tasks \cite{ghorbani2019data,ghorbani2020distributional,wang2025data}. Other papers have argued for the use of other valuation algorithms, like those based on the least core \cite{yan2021if} and influence functions \cite{choe2024your}. 

\subsection{Data Leverage}
Similar concepts have been explored in the study of \emph{data leverage} as conceptualized by \citet{vincent2021data}. Data leverage describes the ways in which data subjects can utilize their data as a way to exert influence over users of that data. Of particular relevance to our work is the concept of a \emph{data strike}, in which users intentionally withold their data from technology companies.\ \citet{vincent2019data} develop a framework for defining different types of data strikes and run simulations to demonstrate how data strikes can impact recommender systems. Our work can be seen as characterizing the efficacy of data strikes where users are witholding their data using deletion requests. Unlike the work of \citet{vincent2019data}, whose results are based on simulations of simple models of deletion requests, our results reflect a worst-case analysis in which deletion requests are chosen to maximize impact on a model.

\section{Modeling Deletion Requests}\label{sec:modeling}
We now present definitions for machine unlearning and deletion requesters. Building on prior work, we define two qualities that a requester may possess, \emph{adaptivity} and \emph{collectivity}, and discuss how those qualities relate to how individuals in the real world might behave when given the ability to delete their data.

\subsection{Notation and Problem Setup}
We assume an arbitrary data domain $\X$ with labels in $\Y$ and a learning algorithm $\alg: (\X\times\Y)^* \rightarrow \Theta$.
An \emph{unlearning algorithm} for $\alg$ is an algorithm $R_\alg:(\X\times\Y)^*\times\Theta\times (\X\times\Y) \rightarrow \Theta$, that maps a labeled dataset $D \in (\X,\Y)^*$, a model $\theta$, and a sample $(x,y) \in D$ to some new model $\theta' \in \Theta$.

We work with labeled datasets $D \in (\X\times\Y)^*$ sampled from some data distribution $\D_f$ which are labeled according to a binary labeling function $f: \X \rightarrow \{0,1\}$. We will often refer to this data distribution as $\D$, where it is assumed that the labeling function $f$ is sampled from some predefined distribution $\F$ over binary functions. For a generic set $S = \{s_1,\ldots,s_n\}$, we use the notation $S_{-i}$ to denote the set containing all but the $i^{th}$ item in $S$.

Adopting the conventions from \citet{ginart2019making}, we present 
a definition for \emph{exact machine unlearning}.
\begin{definition}[Exact Machine Unlearning \citep{ginart2019making}]
     $R_\alg$ satisfies \emph{exact machine unlearning} if, for all $D = \{(x_1,y_1),\ldots,(x_n,y_n)\}$ and $i \in \{1,\ldots,|D|\}$, the random variables $\alg(D_{-i})$ and $R_\alg\big(D,\alg(D),(x_i,y_i)\big)$ are equal in distribution:
    \[\alg\big(D \setminus \{(x_i,y_i)\}\big) =_d R_\alg\big(D,\alg(D),(x_i,y_i)\big).\]
\end{definition}

\subsection{Deletion Requesters}
The object of study in this paper is \textit{deletion requesters}. 
A deletion requester chooses which points are deleted from a model and the order in which those points are deleted. 
Requesters have previously been studied mainly in the context of adaptive machine unlearning \citep{gupta2021adaptive, chourasia2023forget, cohen2026protecting}. 

The requester and the unlearning algorithm take part in a multi-round exchange. At each round (indexed by $t$), the requester outputs a deletion request $u^t \in \X\times\Y$ consisting of a labeled training sample to be deleted. We denote the state of the training dataset after round $t$ as $D^t$, where $D^0 = D$ and $D^t = D^{t-1} \setminus \{u^t\}$. After receiving the unlearning request, the unlearning algorithm outputs the updated state of the model $\theta^t$, such that $\theta^t = R_\alg(D^{t-1},\theta^{t-1},u^t)$ and $\theta^0 = \alg(D^0)$. 

\begin{definition}[Deletion Requesters]
A deletion requester is a stateful mapping $\req: ((\X\times\Y)^*\times\Theta)^+ \rightarrow (\X\times\Y)$ (possibly randomized) that takes as input an initial training dataset $D$ and the history of interaction between itself and the algorithms, and outputs a new datapoint to be deleted in the current round. Given a deletion requester $\req$ and algorithms $(\alg,\RA)$, the deletion sequence $U = \{u^t\}_t$ can be written as
    \begin{align*} 
        u^1 &= \req(D, \theta^0), \\
        u^2 &= \req(D, \theta^0,u^1,\theta^1), \\
        &\ldots\\
        u^t &= \req(D, \theta^0,u^1,\theta^1,\ldots,u^{t-1},\theta^{t-1}). 
    \end{align*}
\end{definition}
We will consider deletion requesters that have a \emph{deletion budget} $k$, the maximum number of deletions it may request.

\subsection{Adaptive and Nonadaptive Requesters}
\emph{Adaptive} deletion requesters are able to make decisions about which deletion requests to issue based off of previous model updates \citep{gupta2021adaptive, chourasia2023forget, cohen2026protecting}. 
Our definition of adaptive and nonadaptive deletion requesters is based on \citet{gupta2021adaptive} with some simplifications.\footnotemark~
Informally, a requester is nonadaptive if its deletion sequence is fixed before the first deletion request is sent.

\begin{definition}[Nonadaptive/Adaptive Requesters]\label{def:adaptive}
    A deletion requester is \emph{adaptive} if it is not \emph{nonadaptive}.

    A deletion requester $\req$ is \emph{nonadaptive} if it is independent of all models after time step 0. That is, if there exists a mapping $\req': (\X\times\Y)^*\times\Theta \rightarrow (\X\times\Y)$ which when given an initial training dataset $D$ and initial model $\theta^0$ yields an ordering of the training datapoints
    \[ (u^1,\ldots,u^t) = \req'(D, \theta_0)\]
    such that, for all $t \geq 1$,
    \[ u^t = \req\big(D, \theta^0,u^1,\theta^1,u^2\ldots,u^{t-1},\theta^{t-1}\big).\]
\end{definition}

\footnotetext{%
    \citet{gupta2021adaptive} allows for insertion and deletion; we consider only deletions. They consider a publishing function that passes information about the model to requesters; our requesters observe the full state of the model at each time step.
    They require nonadaptive requesters to be independent of all models; we allow our nonadaptive requesters to see the initial model $\theta_0$ before fixing the order of deletion requests.
}

\subsection{Individual and Collective Requesters}

A formal definition of individual and collective requesters is presented below.
To model individual requesters, we imagine that each individual $i$ to whom data $(x_i,y_i)$ relates decides when to request erasure independently using some score function $s$. The individual with the highest score requests erasure first.

\begin{definition}[Individual/Collective Requesters]\label{def:collective}
    A requester is \emph{collective} if it is not \emph{individual}.
    
    A deletion requester is \emph{individual} if it can be modeled as taking the highest value datapoints according to a deterministic scoring function $s:(\X\times\Y) \times \Theta^* \rightarrow \reals$ without ties. Concretely, for a given time step $t$ in the adaptive setting, 
    \[u^t = \argmax_{(x,y) \in D^{t-1}}s\Big((x,y),\theta^0,\ldots,\theta^{t-1}\Big),\]
    where $D^{t-1} = D \setminus\{u^1,\ldots,u^{t-1}\}$. In the nonadaptive setting, the sequence $(u_1, \ldots,u_t)$ are the pairs $(x_i,y_i)\in D$ sorted in ascending order by $s\big((x_i,y_i),\theta^0\big)$.
\end{definition}

We prohibit ties in the scoring function of individual requesters to enforce a meaningful distinction between individual and collective requesters. If ties could be broken arbitrarily and $s$ was a constant function, then an individual requester could implement a collective requester through its tie-breaking behavior. A more natural definition might allow for breaking ties randomly. However, this complicates the analysis of the optimal individual requester for reasons we discuss in Appendix~\ref{sec:ind_requesters}. An equivalent definition to ours would be to allow for ties in $s$ but require that ties be broken according to some fixed order (e.g., lexicographically).

\section{Goals of the Learner and Deletion Requester}

We seek to understand the impact that different types of requesters have on \emph{model performance}. We define a model performance metric as a function $g:\Theta \rightarrow \mathbb{R}$ which gives a real valued measurement of some property of a model. 
Of particular interest, and the focus in this paper, is expected error.
For a model $\theta: (\X\times\Y) \rightarrow \{0,1\}$, error is measured as
\[\err(\theta) = \expect_{(x,y) \sim \D}\mathbbm{1}[\theta(x) \neq y].\]
Given a distribution $\D$, the goal of an optimal learning algorithm 
will therefore minimize the following term:
\[\err^*(\D) :=\inf_\alg\expect_{D\sim\D} \err(\alg(D)).\]

Works in the machine unlearning literature tend to measure aspects of model performance before and after unlearning as a way to evaluate the suitability of proposed methods. The most common metric of this sort is test accuracy, which is largely standard across evaluations, but other metrics like group fairness~\citep{oesterling2024fair,chen2023fast} or empirical privacy~\citep{carlini2022privacy,chen2021machine} have also been studied.

What is considered less often in prior work is the connection between the deletion requester and measurements of model performance after unlearning. Different distributions of unlearning requests can lead to vastly different performance in downstream models. For example, unlearning the entirety of a class will likely degrade test accuracy much more than unlearning uniformly distributed training samples. 
It is therefore natural to characterize deletion requesters by the degree to which they alter model performance.

In order to study the worst-case effects of different classes of requesters, we give the requester an adversarial goal: degrade performance as much as possible.
We define a requester's \emph{utility} after $t$ rounds as the difference in performance measure $g$ between $\theta^0$ and $\theta^t$, where $\theta^t$ is the state of the model after $t$ rounds of unlearning.
\begin{definition}[Requester Utility]
    Let $D$ be a training dataset, $\theta^0 := \alg(D)$ be an initial model. Let $\theta^t$ be the state of the model after $t$ rounds of interaction between a deletion requester and a model trainer. A requester $\req$'s utility at time step $t$ with respect to $g:\Theta \to \reals$ (and implicitly with respect to $\alg$ and $R_\alg$) is defined as 
    \[\util_{\req,g}(\theta^0,D,t) = g(\theta^t) - g(\theta^0).\]
\end{definition}
Requester utility is defined as a random variable with randomness over $\alg$, $R_\alg$, and $\req$. At times, we will also reason about the \emph{expected requester utility}:
\begin{definition}[Expected Requester Utility]
    Let $\D$ be a be a distribution on training datasets in $(\X\times\Y)^*$. A requester $\req$'s expected utility at time step $t$ with respect to $g:\Theta \to \reals$ is defined as 
    \[\overline{\util}_{\req,g}(\D,t) = \expect_{\substack{f\sim\F\\D\sim\D_f}} \util_{\req,g}(\alg(D),D,t).\]
\end{definition}

We seek to characterize the abilities of \emph{optimal requesters}: those that maximize expected requester utility. For a given performance function $g$, we define the optimality of a deletion requester with respect to the data distribution $\D$, and number of deletions $k$: 
\begin{align}
    \opt_g(\D, k) = \sup_{\req}\;\overline{\util}_{\req,g}(\D,k).
\end{align}
This optimum is also dependent on $\alg$ and $R_{\alg}$, although we omit them from the notation for the sake of brevity.

\subsection{Adaptivity and Collectivity Gaps}
Our work focuses on gaps between the utility of adaptive and nonadaptive deletion requesters, as well as between collective and individual deletion requesters. Many prior works have studied similar gaps between different classes of algorithms. We present a definition for an \emph{adaptivity gap} similar to that of \citet{dean2008approximating}, who studied the impacts of adaptivity on the stochastic knapsack problem.
\begin{definition}[Adaptivity Gap \citep{dean2008approximating}]
    Let $\mathsf{Adapt}(I)$ be the  supremal expected value achieved by any adaptive policy 
    on a problem instance $I$. Let $\mathsf{Nonadapt}(I)$ be the supremal expected value achieved by any nonadaptive policy 
    on a problem instance $I$. The \emph{adaptivity gap} is defined as 
    \[\sup_{I} \frac{\mathsf{Adapt}(I)}{\mathsf{Nonadapt}(I)},\]
    where the supremum is taken over all problem instances.
\end{definition}

We also introduce what we call the \emph{collectivity gap}: the ratio between the expected utility achieved by the optimal collective requester and optimal individual requester.
\begin{definition}[Collectivity Gap]
    Let $\mathsf{Coll}(I)$ be the supremal expected value achieved by any collective policy on a problem instance $I$. Let $\mathsf{Ind}(I)$ be the supremal expected value achieved by any individual policy on a problem instance $I$. The \emph{collectivity gap} is defined as 
    \[\sup_{I} \frac{\mathsf{Coll}(I)}{\mathsf{Ind}(I)},\]
    where the supremum is taken over all problem instances.
\end{definition}
In this work we define these gaps over all instances of possible unlearning task, making use of the definitions of adaptivity and collectivity we presented in Section \ref{sec:modeling}. 

\section{Characterizing Unlearning Requesters}\label{sec:learning}

We want to understand how adaptivity and collectivity affect machine unlearning model performance.
To do so, we will study adaptivity and collectivity gaps.
Our results are summarized in Table~\ref{tab:gaps}.

Let's review the different optimization problems involved in doing so.
A learning problem instance is defined by a data domain $\X$, a label domain $\Y$, a data distribution $\D$, a distribution over labeling functions $\F$, and a perfomance metric $g$.
Given a learning problem, the optimal learning algorithm $\alg$ maximizes performance (i.e., minimizes expected error).
Given a learning algorithm $\alg$ and corresponding unlearning algorithm $\RA$, an optimal deletion requester $\req$ maximizes the degradation in performance (i.e., maximizes the change in expected error) subject to its deletion budget $k$. Deletion requesters of different types may degrade performance to a greater or lesser extent.

Our task, therefore, is to design learning problem instances where the optimal requesters of different types are differently able to degrade model performance, for optimal $\alg$ (and appropriate $\RA$).

That is what we do next.
We instantiate a simple learning task (loosely based on \cite{feldman2020does}) for which it is easy to reason about and characterize the optimal learning algorithms and requesters of different types.

\begin{table*}
    \centering
    \scalebox{1}{
    \begin{tabular}{ccc}
        \toprule
        Gap & Lower Bound & Upper Bound\\
        \midrule
        Individual Adaptive vs. Individual Nonadaptive & $\sim$1.06 & Open \\
        Collective Adaptive vs. Collective Nonadaptive & 1 & 1 \\
        Collective Adaptive vs. Individual Adaptive & 1.5 & Open \\
        Collective Nonadaptive vs. Individual Nonadaptive & 1.5 & Open \\
        \bottomrule
        \vspace{1pt}
    \end{tabular}
    }
    \label{tab:gaps}
    \caption[Gaps]{\textbf{Gaps in our simplified learning setting.} We provide lower and upper bounds for the gaps between optimal requesters of different types in our simplified learning setting. 
    The remaining lower bounds are derived in Section~\ref{subsec:gaps}. 
    To the best of our knowledge, upper bounds on collectivity gaps and the individual adaptivity gap have not been studied in prior work and are not addressed in this paper.
    }
\end{table*}

\subsection{A Simple Learning Problem}
We instantiate a data domain $\X$ consisting of $m$ subpopulations $\X = \{1,2,\dots,m\}$ with binary labels $\Y = \{0,1\}$. We define a distribution of datasets 
$\D$ on $(\X\times\Y)^*$ parameterized by $m$ values $\lambda_1,\lambda_2,\ldots,\lambda_m$ such that $\lambda_i \in \reals^+$. 
To sample a dataset $D$ from $\D$, we first sample a labeling function $f \sim \F$, where $\F$ is the uniform distribution over all binary functions on $\X$. For each $i \in \X$ we sample $n_i \sim \poisson(\lambda_i)$ and add $n_i$ copies of item $(i,f(i))$ to $D$. 
Owing to properties of the Poisson distribution, the expected size of $D$ is $N := \sum_{i\in\X}\lambda_i$. Additionally, since the expectation of $n_i$ is $\lambda_i$, the expected proportion of $D \sim \D$ made up of copies of subpopulation $i$ is $p_i = \frac{\lambda_i}{N}$. 
We refer to $p_i$ as the \textit{population frequency} of subpopulation $i$.\footnote{For the sake of clarity, we tend to parameterize $\D$ by $N$ and $p_1,\ldots,p_m$ rather than $\lambda_1,\ldots,\lambda_m$. These formulations are equivalent.}
We emphasize that $D$ is a multiset, in which a given point $i$ can (and likely will, for large enough $N$) occur multiple times. 
There is no label noise in this model; every sample from the same subpopulation will have the same label.

We now show that in our learning setting, the learning algorithm that simply memorizes the labels of the samples included in the training dataset and randomly guesses\footnotemark~on the rest is optimal. A similar result is proved in \citet[Theorem 2.3]{feldman2020does}. We provide a simplified justification in Appendix~\ref{app:claim}.
\begin{claim}\label{claim:opt-model}
For the problem described above, the minimum population error $\err^*(\D)$ is achieved by the learning algorithm $\alg^*$ which yields a model $\theta^* := \alg^*(D)$ of the form
\begin{align*}
    \theta^*(x) = \begin{cases}
        y &\text{if } (x,y) \in D\\
        \mathbbm{1}[H = 1] &\text{otherwise},
    \end{cases}
\end{align*}
where $H \sim \text{Bernoulli}\left(\frac{1}{2}\right)$.
\end{claim}

\footnotetext{\label{footnote:guessing}Any behavior in the case where $(x,y) \notin D$ will lead to equivalent performance, as the predictions will be independent of the true label. However, adopting this randomized behavior simplifies our analysis in Section~\ref{subsec:gaps}.}

Since the learning algorithm is memorizing training data and there is no label noise, the resulting model will be correct on all the subpopulations it has seen during training and will guess with 50\% accuracy on all the subpopulations it hasn't seen. We can therefore characterize expected population error of a model output by this learning algorithm:
\[\err(\alg^*(D)) = \sum_{i \in [m]} \frac{p_i\cdot\mathbbm{1}[(i,f(i)) \notin D]}{2}.\]

\subsection{Instantiating the Unlearning Problem}
We fix the unlearning algorithm to be retraining from scratch:
\[R_\alg(\theta,D,i) = \retrain_{\alg}(\theta,D,i) = \alg(D_{-i}).\]
Seeing as retraining is the ``gold standard'' for unlearning, this choice allows us to abstract away any imperfections of the unlearning algorithm and focus on the impact of deletion requests on the downstream models. 

As we allude to above, we focus on performance measure $g(\theta) = \err(\theta)$. Therefore, the goal of the optimal unlearning requester is to maximize the population error of the model output after $t$ rounds of unlearning. Within this setting, we can intuitively understand how a requester would achieve this goal. In an interaction with the optimal learning algorithm $\alg^*$, a deletion requester will only change the model if every instance of a subpopulation in the training dataset is unlearned. 
Consider the utility of a requester $\req$ after $k$ rounds of unlearning, letting $U$ be the multiset of $k$ points that are deleted. We denote by $D_i$ the number of occurrences of subpopulation $i$ in $D$, and define $U_i$ similarly.
\ifsubmission
\begin{align*}
    \util&_{\req,\err}(\theta^0,D,k) = \err(\theta^k) - \err(\theta^0)\\
    &= \frac{1}{2}\sum_{i \in [m]} p_i\Big(\mathbbm{1}[(i,f(i)) \notin D\setminus U] - \mathbbm{1}[(i,f(i)) \notin D]\Big)\\
    &= \frac{1}{2}\sum_{i \in [m]}p_i \cdot \mathbbm{1}[U_i > 0 \text{ and } U_i = D_i]).
\end{align*}
\else 
\begin{align*}
    \util_{\req,\err}(\theta^0,D,k) &= \err(\theta^k) - \err(\theta^0)\\
    &= \frac{1}{2}\sum_{i \in [m]} p_i(\mathbbm{1}[(i,f(i)) \notin D\setminus U] - \mathbbm{1}[(i,f(i)) \notin D])\\
    &= \frac{1}{2}\sum_{i \in [m]}p_i \cdot \mathbbm{1}[U_i > 0 \text{ and } U_i = D_i]).
\end{align*}
\fi
This shows us that the utility of a requester in this setting is completely determined by the total population frequency of the subpopulations that are completely removed from $D$. An optimal requester must maximize this value.

\subsection{Optimal Requesters and the Knapsack Problem}
Next we will see that optimal deletion requesters (adaptive or nonadaptive) for our learning problem must solve an instance of the knapsack problem. In the knapsack problem, we are given a list of \emph{items} $1,\ldots,m$, corresponding \emph{weights} $w_1,\ldots,w_m \in \mathbb{R}^+$ and \emph{values} $v_1,\ldots,v_m \in \mathbb{R}^+$ for each item, and a knapsack with a given \emph{capacity} $k \in \mathbb{N}$. The goal is to find the subset of items $S \subseteq[m]$ maximizing $\sum_{i\in S}v_i$ subject to $\sum_{i\in S}w_i \leq k$.

We reiterate that the task of the optimal requester is to increase the population error (thereby decreasing the accuracy) of a model as much as possible. Because of the structure of the optimal model (Claim~\ref{claim:opt-model}), error will only increase if all samples from a subpopulation are removed from $D$. In particular, once a subpopulation is fully removed, the accuracy of a model will decrease by that subpopulation's population frequency. We can therefore restate the task of the unlearning requester: Given a number of unlearning requests $k$ (\emph{capacity}), find the subset of subpopulations for whom the number of times they appear in $D$ (\emph{weights}) sum to no more than $k$ and the sum of whose population frequencies (\emph{values}) is maximized.

It is straightforward to show that an optimal \emph{collective} unlearning requester (either adaptive or nonadaptive) is solving an instance of the knapsack problem. Collective requesters can act on full knowledge of the training dataset and therefore are able to compute the weights corresponding to each subpopulation ($w_i = n_i$). Subpopulation frequencies are properties of the data distribution and accessible to any requester ($v_i = p_i$). Since we are not placing any restrictions on running time, a collective requester can simply compute the optimal solution to the knapsack problem and request the deletion of training samples represented by the subpopulations in the solution.

The connection between \emph{individual} requesters and knapsack is more complex. As stated above, all requesters can make use of subpopulation frequencies, but individual requesters must fix their scoring function $s$ before observing $D$, meaning they cannot act on full knowledge of the weights. Instead, individual requesters can only act on knowledge of the distribution that the weights were drawn from: in this case, a Poisson centered at $\lambda_i$. (For additional discussion on the capabilities of individual requesters, see Appendix~\ref{sec:ind_requesters}.) 

The optimal individual requester turns out to be solving an instance of the \emph{correlated stochastic} variant of the knapsack problem. Instead of being fixed, item weights in the stochastic knapsack problem are independent random variables drawn from known distributions~\citep{rothkopf1966scheduling,dean2008approximating,bhalgat2011improved}I.\footnote{The requirement of the stochastic knapsack problem that item sizes be independently sampled motivated our choice of data distribution, sampling subpopulations sizes from Poisson distributions rather than sampling datasets of a fixed size, as is done in \cite{feldman2020does}.}
An algorithm placing items into the knapsack is only able to see an item's realized weight after it is inserted; once an item overflows the knapsack, execution is terminated. 
An additional subtlety of our problem setting is the dependence between realized item weights and item values (i.e.,  $v_i = 0$ if $n_i = 0$). This corresponds to the \emph{correlated} stochastic knapsack problem \citep{gupta2011approximation}, where each item has an associated probability distribution over (size, reward) pairs $(\pi_{i,t},R_{i,t})$, where $\pi_{i,t}$ represents the probability that item $i$ has instantiated weight $t$ and $R_{i,t}$ is the value of item $i$ if it has instantiated size $t$.

\section{Theoretical Results}\label{subsec:gaps}
We now demonstrate the existence of gaps between individual and collective requesters and nonadaptive and adaptive requesters in our learning model. The lower bounds we present show how requesters that take advantage of collective and adaptive behaviors can have more influence over models after unlearning.

Note that while there are a number of prior works studying adaptivity gaps for stochastic knapsack~\cite{dean2008approximating,gupta2011approximation, barak2026stochastic}, none directly apply to our setting: the instances of stochastic knapsack induced by our learning/unlearning problem have extra structure.
A discussion of these results is included in Appendix~\ref{app:prior_gaps}.

\subsection{Adaptivity Gap (Collective)} 
It is straightforward to show that there will be no gap between collective adaptive and collective nonadaptive requesters when the unlearning algorithm is retraining from scratch. 
\begin{theorem}\label{thm:nogap}
    For all  $\X$, $\Y$, $\D$, $\F$, $\alg$, $g$, and $k$, and for $\RA = \retrain_\alg$ the adaptivity gap for collective requesters is equal to 1. 
\end{theorem}
\begin{proof}
    Let $D\sim\D$ be a random variable representing the initial training dataset. Since the unlearning method is $\retrain_\alg$, the final model $\theta^k$ will be the output of $\alg(D \setminus U_k)$, where $U_k$ is the set of unlearning requests issued by the requester after $k$ rounds of deletion. The optimal set of deletion requests will be that which maximizes the gap in performance between the initial and final model, i.e.
    \[\argmax_{\substack{U \subseteq D\\|U| = k}}\left[g(\alg(D \setminus U)) - g(\theta^0)\right].\]
    The requester which issues requests in this way will maximize expected utility. Notably, this optimization problem can be solved without knowledge of the intermediate model states $\theta^1,\ldots,\theta^{k-1}$. Since all collective requesters can act on full knowledge of the training dataset, both collective adaptive and collective nonadaptive requesters can solve this optimization problem, meaning that the optimal collective adaptive and collective nonadaptive requesters will achieve the same expected utility.
\end{proof}

\subsection{Adaptivity Gap (Individual)}
The intuition behind the existence of an individual adaptivity gap in this setting is simple: having observed the sizes of previously deleted subpopulations, a requester has more information to determine the optimal next subpopulation to remove. However, the parameters that induce an adaptivity gap for individual requesters in this setting are not obvious. 
Nevertheless, we now show that non-trivial adaptivity gaps can exist for individual requesters.

For simplicity, we will focus on the case of three subpopulations ($\X = \{1,2,3\}$). In this setting, the optimal nonadaptive requester $\mathtt{Nonadapt}$ outputs an ordering of the subpopulations $a,b,c$ that maximizes expected utility. The utility of this sequence is defined by the following expression:
\ifsubmission
\begin{align}
    \begin{split}
    \overline{\util}&_{\mathtt{Nonadapt},\err}(\D,k)\\ 
    &=\ p_a\cdot\Pr[n_a > 0 \text{ and } n_a \leq k]\\
    &\quad+ p_b\cdot\Pr[n_b > 0 \text{ and } n_a + n_b \leq k]\\
    &\quad+ p_c\cdot\Pr[n_c > 0 \text{ and } n_a + n_b + n_c \leq k]. 
    \end{split}\label{eq:nonadapt_util}
\end{align}
\else
\begin{align}
    \begin{split}
    \overline{\util}_{\mathtt{Nonadapt},\err}(\D,k) =\ &p_a\cdot\Pr[n_a > 0 \text{ and } n_a \leq k]\\
    &+ p_b\cdot\Pr[n_b > 0 \text{ and } n_a + n_b \leq k]\\
    &+ p_c\cdot\Pr[n_c > 0 \text{ and } n_a + n_b + n_c \leq k]. 
    \end{split}\label{eq:nonadapt_util}
\end{align}
\fi

Our goal is to find an adaptive requester which has higher utility than the optimal nonadaptive requester. We will consider a restricted class of adaptive requesters which we call \textit{thresholded adaptive requesters}. 
We define the \emph{thresholded adaptivity gap} to be the gap between the optimal nonadaptive requester and the optimal thresholded adaptive requester under a certain set of parameters.
The adaptivity gap is at least as large as the threshold adaptivity gap.

\begin{definition}[Thresholded Adaptive Requesters]
Given an integer $t$, \emph{thresholded adaptive requesters} are a class of requesters denoted by $\langle a,b,c\mid t \rangle$, where $a$, $b$, and $c$ are distinct subpopulations in $\X$. 
Such a requester would first request the deletion of all instances of subpopulation $a$. If $n_a > t$, the requester will then request all instances of subpopulation $b$, followed by all instances of subpopulation $c$. Otherwise, if $n_a \leq t$, the requester will then request all instances of subpopulation $c$, followed by all instances of subpopulation $b$.

\end{definition}

We now present the expected utility of $\langle a,b,c\mid t \rangle$:
\ifsubmission
\begin{align}
    \begin{split}
    &\overline{\util}_{\langle a,b,c\mid t \rangle,\err}(\D,k) \\
    &=\ p_a\cdot\Pr[n_a > 0 \land n_a \leq k]\\
    &+ \Pr[n_a > t] \cdot 
        \biggl(p_b\cdot\Pr[n_b > 0 \land n_a + n_b \leq k \mid n_a > t]\\
        &\quad\quad + p_c\cdot\Pr[n_c > 0 \land n_a + n_b + n_c\leq k \mid n_a > t]\biggr)\\
    &+ \Pr[n_a \leq t] \cdot 
        \biggl(p_c\cdot\Pr[n_c > 0 \land n_a + n_c \leq k \mid n_a \leq t]\\
        &\quad\quad+ p_b\cdot\Pr[n_b > 0 \land n_a + n_b + n_c\leq k \mid n_a \leq t]\biggr).
    \end{split}\label{eq:adapt_util}
\end{align}
\else
\begin{align}
    \begin{split}
    \overline{\util}_{\langle a,b,c\mid t \rangle,\err}(\D,k) =\ &p_a\cdot\Pr[n_a > 0 \land n_a \leq k]\\
    &+ \Pr[n_a > t] \cdot (p_b\cdot\Pr[n_b > 0 \land n_a + n_b \leq k \mid n_a > t]\\
    &\phantom{+ \Pr[n_a > t] \cdot (}+ p_c\cdot\Pr[n_c > 0 \land n_a + n_b + n_c\leq k \mid n_a > t])\\
    &+ \Pr[n_a \leq t] \cdot (p_c\cdot\Pr[n_c > 0 \land n_a + n_c \leq k \mid n_a \leq t]\\
    &\phantom{+ \Pr[n_a > t] \cdot (}+ p_b\cdot\Pr[n_b > 0 \land n_a + n_b + n_c\leq k \mid n_a \leq t]).
    \end{split}\label{eq:adapt_util}
\end{align}
\fi

We construct an instance of our learning problem with (threshold) adaptivity gap at least $1.064$. 
\begin{theorem}\label{thm:adaptivity_gap}
    The adaptivity gap for individual requesters is at least 1.064.
\end{theorem}
\begin{proof}
    Let $\X = \{1,2,3\}$ and $\Y = \{0,1\}$. Let $\D$ have parameters $\lambda_1 = 0.41$, $\lambda_2 = 0.42$, $\lambda_3 = 0.18$, and $N = 50$, and let $\F$ be the uniform distribution on $\Y^\X$. Let $g = \err$ and $k = 40$. The ratio between the utility of the optimal individual adaptive and individual nonadaptive requesters under these parameters is approximately 1.064.
\end{proof}
To find this instance, we first fixed $N=10$ and $k=8$ and performed a grid search over parameters $p_1 + p_2 + p_3 = 1$ with a step size of $0.01$. For each choice of $(p_1,p_2,p_3)$, we explicitly computed the expected utility of the optimal nonadaptive requester~\eqref{eq:nonadapt_util}, as well as the expected utility of the all thresholded adaptive requesters of the form $\langle a,b,c \mid t \rangle$~\eqref{eq:adapt_util} where $(a,b,c)$ is a permutation of $(1,2,3)$ and $0 \leq t \leq k$. 
Figure~\ref{fig:heatmap} in Appendix~\ref{app:viz} plots the thresholded adaptivity gap over the parameters of our grid search, with a maximum of about $1.011$.

Observing that the highest values are achieved when two of the subpopulation frequencies are approximately equal and larger than the third probability, we performed a further grid search over probabilities of the form $p_1 = p_2 = \epsilon$, $p_3 = 1 - 2\epsilon$. The results are in Figure~\ref{fig:square_heatmap} in Appendix~\ref{app:viz}. For this search, we varied $\epsilon$  between $0.33$ and $0.50$ with a step size of $0.01$, varied integer values of $N$ between 10 and 50, and fixed and $k = \lfloor0.8N\rfloor$. We observe a maximum thresholded adaptivity gap of approximately $1.064$

\subsection{Collectivity Gap (Nonadaptive)}
The separation between individual nonadaptive
and collective nonadaptive requesters can be shown through a simple example. Here's the intuition. Consider the case where there are two subpopulations, each of which has a population frequency $p_i = \frac{1}{2}$ (i.e. $\X = \{1,2\}$, $\D$ has parameters $\lambda_1 = \lambda_2 = \frac{N}{2}$). The counts of each subpopulation in our training dataset will be distributed as Poissons, as described in Section~\ref{sec:learning}: 
\(n_i \sim \poisson\left(\lambda_i\right).\)
Let $k = \frac{N}{2}$. Exactly one of the $n_i$ will be less than $k$ with probability about $1/2$. When this happens, the only way to affect the model is to remove the smaller subpopulation. A collective requester knows which subpopulation is smaller while an individual requester can only guess.

\begin{theorem}\label{thm:nonadaptive_collective}
    For any $\epsilon>0$, the collectivity gap for nonadaptive requesters is at least $\frac{3}{2}-\epsilon$.
\end{theorem}
\begin{proof}
Let $\X = \{1,2\}$ and $\Y = \{0,1\}$. Let $\D$ have parameters $\lambda_1 = \lambda_2 = \frac{N}{2}$ and $\F$ be the uniform distribution on $\Y^\X$. Let $g = \err$ and $k = \frac{N}{2}$.
We will show that as $N \to \infty$ the collectivity gap approaches $3/2$.
Throughout this proof, we use $\to$ to denote limits as $N\to\infty$.

For large $N$, $\poisson\left(\frac{N}{2}\right)$ approaches a Gaussian centered at $\frac{N}{2}$. In particular, the following all hold: $\Pr[n_1 > k] \to 0.5$, $\Pr[n_2>k] \to 0.5$, $\Pr[n_1 = 0] \to 0$, $\Pr[n_2 = 0] \to 0$, and $\Pr[n_1 + n_2 \leq k] \to 0$.

To determine a lower bound on the nonadaptive collectivity gap, we will compute the values of both $\overline{\util}_{\mathtt{Coll},\err}(\D,k)$ and $\overline{\util}_{\mathtt{Ind},\err}(\D,k)$, where $\mathtt{Coll}$ is the optimal collective requester and $\mathtt{Ind}$ is the optimal individual requester (both nonadaptive). For a requester $\req$, the expected utility is defined as
\ifsubmission
\begin{align*}
    \overline{\util}&_{\req,\err}(\D,k) = \expect_{D \sim \D} [\err(\theta^k) - \err(\theta^0)]\\
    &= \expect_{D \sim \D} \err(\alg(D \setminus U_\req))
    - \expect_{D \sim \D} \err(\alg(D)),
\end{align*}
\else  
\begin{align*}
    \overline{\util}_{\req,\err}(\D,k) &= \expect_{D \sim \D} [\err(\theta^k) - \err(\theta^0)]\\
    &= \expect_{D \sim \D} \err(\alg(D \setminus U_\req)) - \expect_{D \sim \D} \err(\alg(D)),
\end{align*}
\fi
where $U_\req$ is the set of $k$ deletion requests issued by $\req$. We therefore need to compute $\expect_{D \sim \D} \err(\alg(D))$, $\expect_{D \sim \D} \err(\alg(D \setminus U_\mathtt{Coll}))$, and $\expect_{D \sim \D} \err(\alg(D \setminus U_\mathtt{Ind}))$

First, we look at the expected error of the optimal learning algorithm on the original training dataset. Since it is very likely that both subpopulations appear in $D$, and since the optimal learning algorithm memorizes labels in the training dataset, we can conclude that
\[\expect_{D \sim \D}[\err(\alg^*(D))] \to 0.\]

A collective requester $\mathtt{Coll}$, having full knowledge of $n_1$ and $n_2$, would be able to choose the subpopulation corresponding to the smaller $n_i$ and request its deletion first, followed by the other subpopulation. The probability that $\mathtt{Coll}$ will be able to delete the entirety of the first subpopulation it removes is equal to the probability that $n_1 \leq k$ or $n_2 \leq k$, which is $\Pr[n_1 \leq k \text{ or } n_2 \leq k] = 1 - \Pr[n_1 > k \text{ and } n_2 > k] \to 1 - \frac{1}{4} = \frac{3}{4}$.  The probability that $\mathtt{Coll}$ will be able to delete the entirety of both subpopulations is equal to the probability that $n_1 + n_2 \leq k$, which as stated above approaches 0. If $\mathtt{Coll}$ is able to delete a subpopulation, then the expected error will approach $\frac{1}{4}$, since the optimal model will now have to guess with 50\% accuracy on that subpopulation, which has a population freqency of $\frac{1}{2}$. If it is not, then the expected error will be 0 since the correct labels for both subpopulations are still present in the updated training dataset. 
By the law of total expectation,
\[\expect_{D \sim\D}[\err(\alg^*(D\setminus U_\mathtt{Coll}))] \to \frac{3}{4} \cdot \frac{1}{4} + \frac{1}{4} \cdot 0 = \frac{3}{16}.\]

An individual requester $\mathtt{Ind}$, on the other hand, cannot not know which of $n_1$ and $n_2$ is smaller. It can do no better than guessing which class to delete. In the case where $n_1$ and $n_2$ are  both greater than $k$ (occurring with probability approaching $\frac{1}{4}$), the expected error of the model after unlearning will be 0. In the case where $n_1$ and $n_2$ are both less than or equal to $k$ (w.p.\ approaching $\frac{1}{4}$), the expected error of the model after unlearning will be $\frac{1}{4}$. In the case where one of $n_1$ and $n_2$ is less than or equal to $k$ and the other is greater (w.p.\ approaching $\frac{1}{2}$), $\mathtt{Ind}$ will choose a random subpopulation to delete and be correct with probability approaching $\frac{1}{2}$.  The expected error of the model after unlearning will be $\frac{1}{2} \cdot \frac{1}{4} = \frac{1}{8}$. 
Then, 
\[\expect_{D \sim\D}[\err(\alg^*(D\setminus U_\mathtt{Ind}))] \to \frac{1}{4} \cdot \frac{1}{4} + \frac{1}{2} \cdot \frac{1}{8} + \frac{1}{4} \cdot 0 = \frac{1}{8}.\]

Putting it all together, we can compute that the expected utilities of the optimal nonadaptive collective and individual requesters:
\[\overline{\util}_{\mathtt{Coll},\err}(\D,k) \to \frac{3}{16} 
\quad \quad
\overline{\util}_{\mathtt{Ind},\err}(\D,k) \to \frac{1}{8} \]
Therefore, for any $\epsilon>0$ the collectivity gap for unlearning is as least $\frac{\overline{\util}_{\mathtt{Coll},\err}(\D,k)}{\overline{\util}_{\mathtt{Ind},\err}(\D,k)} = \frac{3}{2}-\epsilon$.
\end{proof}

\subsection{Collectivity Gap (Adaptive)}
Next, we show that the same instance that gives the nonadaptive collectivity gap in Theorem~\ref{thm:nonadaptive_collective} also gives an \emph{adaptive} collectivity gap. 
To prove it, we show that the optimal strategies for adaptive requesters (both collective and individual) are the same as their corresponding nonadaptive variants.
\begin{theorem}\label{thm:adaptive_collective}
For any $\epsilon > 0$, the collectivity gap for adaptive requesters in this setting is at least $\frac{3}{2}-\epsilon$.
\end{theorem}
\begin{proof}
    We assume the same two-subpopulation setting as in the proof of Theorem~\ref{thm:nonadaptive_collective}. As shown by Theorem~\ref{thm:nogap}, the optimal collective adaptive and collective nonadaptive requesters in this setting achieve the same utility. Therefore, we only need to compute the optimal utility of an individual adaptive requester.

    Suppose for contradiction that an optimal requester $\req$ begins by issuing deletion requests for subpopulation $a$ and at some time step $1\le t< k$ (before $a$ is completely deleted) preempts $a$ and begins issuing requests for subpopulation $b$. That would imply that at time step $t$, the probability of completely removing subpopulation $b$ was higher than completely removing $a$. In other words, 
    \[\Pr[n_a \leq k | n_a > t] < \Pr[1\le n_b \leq k - t | n_a > t].\] 
    But the above is false. Subtracting these two values, we get
    \begin{align*}
        \Pr[n_a \leq k &| n_a > t] - \Pr[1\le n_b \leq k - t | n_a > t] \\
        &= \frac{\Pr[t < n_a \leq k]}{\Pr[t < n_a]} - \Pr[1\le n_b \leq k - t]\\
        &>\Pr[t < n_a \leq k] - \Pr[0< n_b \leq k - t]\\
        &> 0.
    \end{align*}
    The equality is by independence of $n_a$ and $n_b$. 
    The final inequality is due to the fact that the PMF of a Poisson distribution centered at $k$ is increasing on $[0,k]$. 
    
    Therefore,  at no time $1\le t<n_a$ before the first subpopulation is deleted will it be optimal to preempt and begin deleting the second subpopulation. The optimal strategy for an individual adaptive requester is therefore nonpreemptive, so the collectivity gap will be the same for both adaptive and nonadaptive requesters. From Thm.~\ref{thm:nonadaptive_collective}, the gap is thus $\frac{3}{2}-\epsilon$, for $\epsilon>0$.
\end{proof}

\section{Adaptivity and collectivity in practice}
Our analysis in 
Section~\ref{sec:learning} abstracts away many of the properties 
that make realistic learning tasks difficult, like complex data distributions 
and label noise.
Although designing \emph{optimal} unlearning requesters in a realistic learning 
setting is likely intractable, there are a variety of heuristic methods that can 
be used to outperform random requesters. In this section, we adopt methods from 
the data valuation literature to design requesters that cause significant 
degradation in image classification models. The requesters we design help explain how and when requesters will outperform random deletion requests, as well as help clarify the impacts of adaptivity and collectivity.

\subsection{Data Valuation-Based Requesters}
\label{sec:shapley}
Individuals may be motivated by the \emph{value} of their data when requesting deletion. If an individual's data is very important to a model's decisions, the user may feel that they deserve compensation. If their data is unimportant, they may care less about their presence in the training data. We look to the data valuation literature to codify this behavior.
We build a requester around the Data Shapley value, a valuation method that has seen varied 
applications in machine learning~\cite{ghorbani2019data}.

\begin{definition}[Data Shapley Value]\label{def:shap}
    Let $D = \{(x_j,y_j)\}_{j \in [n]}$ be a dataset and 
    $\fn: (\X\times\Y)^* \rightarrow \mathbb{R}$ be an arbitrary utility 
    function mapping datasets to scores. The data Shapley value for the $j^{th}$ point in $D$ is defined as 
    \[\phi_j(D,\fn) = \frac{1}{n}\cdot\sum_{S \subseteq D \setminus\{(x_j,y_j)\}}\frac{\fn(S \cup\{(x_j,y_j)\}) - \fn(S)}{\binom{n - 1}{|S|}}.\]
\end{definition}
When it is clear from context, we will use $\phi_j$ to refer to $\phi_j(D,\fn)$. 
The utility function $\fn$ in Definition~\ref{def:shap} can represent an arbitrary function of a dataset, such as the expected accuracy of a model trained on that dataset.
Informally, $\phi_j$ is one way to measure how much value $(x_j,y_j)$ contributes to $\fn(D)$. 
We utilize this insight to design the Shapley requester, which removes points ordered by their value to some utility function.
\begin{definition}[Shapley Requester]
    For each index $i \in [n]$, the Shapley requester $\shap$ computes $\phi_i(D,\fn)$ and returns the samples in $D$ ordered by decreasing values of $\phi_i$, with ties broken randomly.
\end{definition}

There is reason to believe that the Shapley requester is an effective unlearning requester in practice. In the data valuation literature, techniques are often evaluated by removing points from a dataset ordered by valuation and retraining models on the progressively smaller datasets (see \citet[Figure 2]{ghorbani2019data}, \citet[Figure 1]{yan2021if}). These evaluations show that ordering deletions according to valuation metrics like the Shapley value can resulting a more rapid degradation in model performance than random orderings, indicating a positive correlation between a data point's value and its contribution to unlearning utility.

As we have defined it, the Shapley requester is \emph{nonadaptive} and \emph{collective}; the ordering of deletions is determined based only on the initial dataset, and computing a Shapley value with respect to a dataset $D$ requires full knowledge of $D$. We make minor modifications to this requester in order to isolate the effects of adaptivity and collectivity on requester performance. 

First, we define the \emph{adaptive Shapley requester}. 
The adaptive Shapley requester recomputes the Shapley values after each deletion with respect to all of the remaining points in the training dataset.

\begin{definition}[Adaptive Shapley Requester]
    The order of deletions requested by the adaptive Shapley requester $u_1,\ldots,u_k$ is captured by the following recurrence relation:
    \begin{align*}
        u_1 &= \argmax_{j\in [n]} \phi_j(D,\fn)\\
        u_t &= \argmax_{j \in [n+1-t]} \phi_j(D^t,\fn),
    \end{align*}
    where $D^t = D \setminus \{x_1,\ldots,x_{t-1}\}$.
\end{definition}

Next, we define the \emph{individual Shapley requester}. 
The individual Shapley requester computes Shapley values with respect to a dataset sampled independently but from the same distribution as $D$. 

\begin{definition}[Individual Shapley Requester]
    Let $D'$ be a dataset drawn from the same distribution as $D$. For each index $j \in [n]$, the individual Shapley requester computes $\phi_j(D' \cup \{(x_j,y_j)\},\fn)$ and returns the samples in $D$ ordered by decreasing values of $\phi_i(D' \cup \{(x_j,y_j)\},\fn)$.
\end{definition}

The individual Shapley requester is inherently \emph{nonadaptive}: the computation of $\phi_j(D'\cup \{x_j,y_j\},\fn)$ depends neither on prior unlearning requests nor the state of the model. This formulation is similar in principle to the distributional Shapley value introduced by \citet{ghorbani2020distributional}.  
We further explain our design of the individual Shapley requester in Appendix~\ref{app:ind_shap}. 

In this work, we will generally take $\fn(S)$ to be the test accuracy of a model trained from scratch on $S$. With this choice of $\fn$, the goal of the Shapley requester and its variants can be thought of as greedily maximizing the expected decrease in test accuracy after each deletion. Observe how this corresponds to the requester's utility when the performance metric $g$ is chosen to be test error. Ordering deletion requests by decreasing Shapley value can thus be seen as a greedy heuristic for assembling a set of training samples which minimizes test accuracy after being unlearned.

We note, however, that the optimization problem that the Shapley requester is solving is distinct from that of an optimal unlearning requesters. In Appendix~\ref{app:shap_subopt}, we prove that there are situations in which the Shapley requester is suboptimal.
\begin{claim}\label{claim:shap_subopt}
    The nonadaptive Shapley requester is not an optimal deletion requester in the learning setting defined in Section~\ref{sec:learning}.
\end{claim}

\subsection{Other Deletion Requesters}
As points of comparison, we consider three other requesters which represent different motivations that users may hold when deleting their data.

\paragraph{The Loss Requester} An individual may simply want a model to not produce correct information about them. In this case, the individual would be more likely to request deletion if the model confidently classifies their data correctly, indicated by low cross-entropy loss. Therefore, the \emph{loss requester} issues requests ordered by decreasing loss (denoted by $\ell(\cdot,\cdot)$).
\begin{definition}[Loss Requester]
    The \emph{nonadaptive loss requester} computes the loss $\ell(\theta^0(x),y)$ for each point $(x,y) \in D$ and returns the samples in $D$ ordered by increasing values of loss.
    \noindent The \emph{adaptive loss requester} issues deletions requests 
    $u_t = \argmax_{(x,y) \in D^j} \ell(\theta^{j-1}(x),y).$
\end{definition}

\paragraph{The MIA Requester} An individual may be motivated by privacy concerns presented by a model more broadly. Rather than their data being regurgitated, this individual is worried that their data has any sort of measurable influence on the model in such a way that might be detected by an attacker. This is the type of privacy leakage that is typically measured using membership inference attacks (MIAs).
Therefore, the \emph{MIA requester} runs a membership inference attack on each training sample and issues requests by decreasing MIA confidence. 
\begin{definition}[MIA Requester]
    Let $\mathtt{MIA}:\Theta \times(\X \times \Y) \to \reals$ be a \emph{membership inference attack} mapping a model and a data sample to a real number.
    The \emph{nonadaptive MIA requester} computes the value $\mathtt{MIA}(\theta^0,(x,y))$ for each point $(x,y) \in D$ and returns the samples in $D$ ordered by increasing values of MIA values.
    \noindent The \emph{adaptive MIA requester} issues deletions requests 
    $u_t = \argmax_{(x,y) \in D^j} \mathtt{MIA}(\theta^{j-1},(x,y)).$
\end{definition}
This requester also reflects the order in which  points were removed in the experiments of \citet{carlini2022privacy}, which highlighted unintended impacts of unlearning sensitive datapoints.

\paragraph{The Random Requester} As a baseline, we implement a requester which issues deletion requests in a random order. This reflects the evaluation methods used for many proposed machine unlearning methods, which sample forget sets uniformly at random from a model's training set. This requester thus reflects the most common approach for evaluating unlearning algorithms and will serve as the baseline in our exploration of requester utility.
\begin{definition}[Random Requester]
    The random requester $\rand$ returns $u_1,\ldots,u_k$ sampled uniformly at random from $D$ without replacement.
\end{definition}

We note that all three of these requesters are \emph{individual}. Neither computing the cross-entropy loss for a datapoint nor computing a LiRA score require knowledge of the particular samples in the training dataset beyond the point being tested.

\subsection{Experimental Setup}
\label{subsec:setup}
\paragraph{Evaluation methods} As mentioned in Section~\ref{sec:shapley}, a common evaluation technique for data valuation methods is to sort points in the training dataset in decreasing order by valuation and, for some subset of $k \leq |D|$, computing the performance of a model trained on the $|D|-k$ least valuable datapoints. These are plotted as series for different valuation metrics to allow for comparisons between metrics \citep{ghorbani2019data,ghorbani2020distributional,yan2021if}. Instead of sorting by valuation according to a metric, we sort training samples according to the order in which a requester requests their deletion. This applies to both nonadaptive requesters, where the request order is fixed at the beginning of the interaction, and adaptive requesters, which recompute orderings in response to each model update. The choice of unlearning algorithm can also be varied. In the valuation literature, retraining from scratch on the $|D|-k$ least valuable points is evaluated, while we can alter this to use any algorithm to unlearn the first $k$ points that are requested.

We call these visualizations in this context \emph{unlearning utility plots}. In our experiments, we typically run 10 rounds of unlearning, requesting points in batches of 1,000. We report average accuracies over 4 trials and give error bars representing one standard deviation above and below the mean.

\paragraph{Datasets and Architectures} All of our experiments make use of the ResNet18 architecture trained on the CIFAR-10 dataset. Rather than train on the full set of 50,000 training samples, we train our initial models on a randomly sampled 50\% subset of CIFAR-10's training set. We use the held-out training samples for the computation of Shapley values in the individual Shapley Requester.

\paragraph{Requesters}
As described above, we run experiments on the Shapley requester, the loss requester, the MIA requester, and the random requester. Additionallly, we run experiments on the adaptive and individual variants of the Shapley requester, as well as the adaptive variants of the loss and MIA requesters.

Since it is computationally infeasible to compute exact Shapley values for a dataset as large as CIFAR-10, we make use of a method for computing approximate Shapley values. In particular, we implement a $k$-nearest neighbors based approach for computing Shapley values first proposed by \citet{jia2019efficient}. This involves training a KNN model on the activations induced by each training sample at the final fully connected layer of the network (i.e. the input to the final softmax layer). We then compute exact Shapley values for the KNN model, which can be computed efficiently using dynamic programming. In our experiments, we set $k = 10$, which we found yielded KNN classifiers with similar performance to the ResNet18 networks being approximated.

In our implementation of the MIA requester, we adopt LiRA \citep{carlini2022membership} as a membership inference attack. To compute LiRA scores, we train 500 shadow models, each of which is trained on a random subset of 25,000 CIFAR-10 training samples. 

\paragraph{Unlearning Methods} 
We make use of three unlearning methods which have been proposed in the literature: Retraining, Relabeling, and SCRUB. Retraining is an exact unlearning method, while relabeling and scrub are approximate unlearning methods. We do not claim that these algorithms necessarily represent the state of the art as it currently stands. Instead, we use them as examples to demonstrate how the utility of an unlearning requester can change depending on the choice of unlearning algorithm

\emph{Retraining} serves as the ``gold standard'' baseline for machine unlearning and simply involves training a new model from scratch on the set of points being retained.

\emph{Relabeling} \citep{golatkar2020eternal} fine-tunes the model on the original training set, except the set of points being forgotten are relabeled with random labels.

\emph{SCRUB} \citep{kurmanji2023towards} uses a student-teacher knowledge distillation approach. The unlearned model (student) is trained in such a way that it behaves similarly to the original model (teacher) but is uncorrelated with the original model on the forget set.

\subsection{Results}
\subsubsection{Comparing Deletion Requesters}
\begin{figure}[ht]
    \centering
    \includegraphics[width=\linewidth]{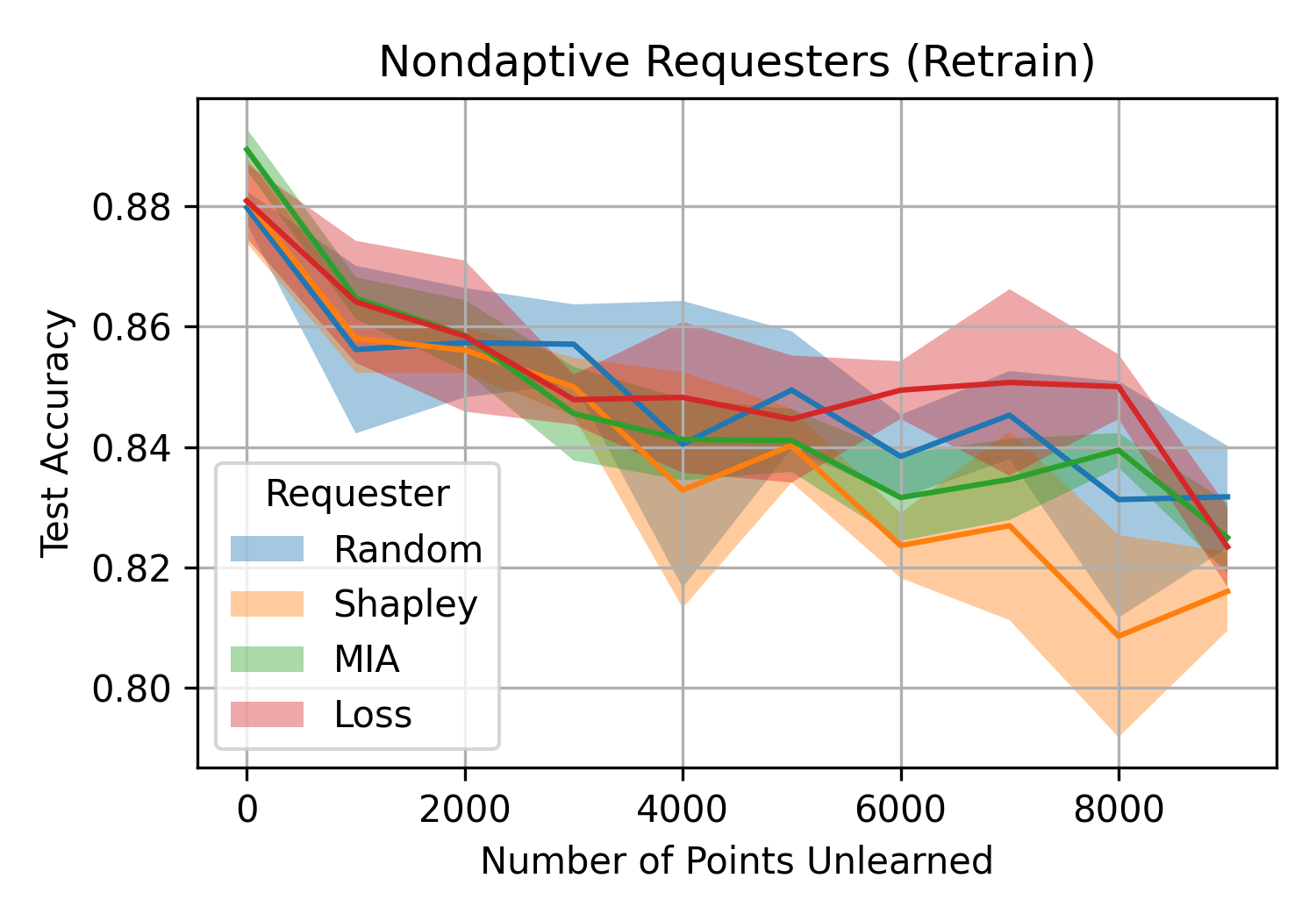}
    \caption[Nonadaptive Requester Unlearning Performance.]{\textbf{Nonadaptive Requester Unlearning Performance.} We plot the performance of the loss, MIA, and Shapley requesters compared to the random requester. For each requester, we train an initial model on a random subset of 25,000 CIFAR-10 training points and use that model to compute an ordering in which the training sampled will be unlearned. Points are then unlearned in this order over 9 iterations of 1,000 points per iteration. Solid lines represent averages over 4 trials, and one standard deviation is filled above and below the mean.}
    \label{fig:nonadaptive}
\end{figure}
We begin by analyzing the nonadaptive variants of the deletion requesters described above (random, loss, MIA, and Shapley). An unlearning utility plot is shown in Figure \ref{fig:nonadaptive} using retraining as the unlearning method. We observe several trends. Perhaps unsurprisingly, all requesters degrade the performance of the model. After the removal of 9,000 data points (36\% of the original training dataset), all methods on average reduce test accuracy to below 84\%, a reduction of over 4 percentage points below the original model. While all of the requesters are relatively close in performance, the Shapley performs the best in the end, achieving a final test accuracy of 81.6\%. For comparison, the random requester achieves a final test accuracy of 83.2\%. This is directionally in line with results from the data valuation literature \cite{ghorbani2019data}, although those results would suggest a larger difference between the Shapley and random requesters. We believe this discrepancy is due to our use of an approximation of the Shapley value (as described in \ref{subsec:setup}).

\subsubsection{Disentangling Adaptivity and Collectivity}
\begin{figure}[ht]
    \centering
    \includegraphics[width=\linewidth]{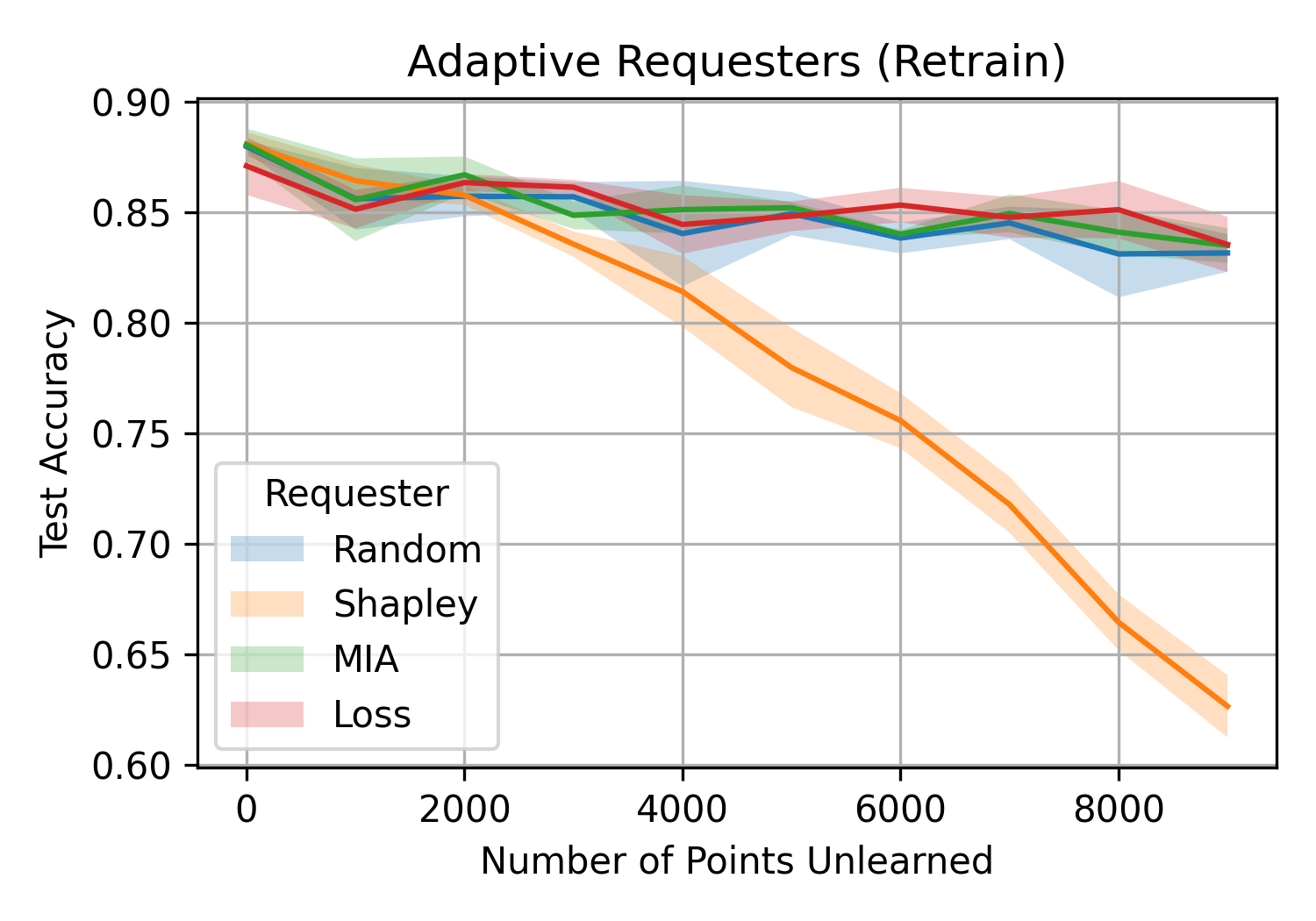}
    \caption[Adaptive Requester Unlearning Performance.]{\textbf{Adaptive Requester Unlearning Performance.} We plot the performance of the adaptive variants of the loss, MIA, and Shapley requesters compared to the random requester.}
    \label{fig:adaptive}
\end{figure}
We now measure the performance of the adaptive and individual variants of the Shapley requester. Results are presented in Figure \ref{fig:adaptive}. 
We observe that  the adaptive loss and MIA requesters perform nearly identically to their nonadaptive variants. However, the adaptive Shapley requester is much more powerful than its nonadaptive variant, reducing test accuracy to 62.7\% on average. 
We hypothesize that the Shapley requester benefits from adaptivity because it attempts to rank points by how important they are to the accuracy of the model, which is exactly the metric that we are measuring here.  

\begin{figure}[ht]
    \centering
    \includegraphics[width=\linewidth]{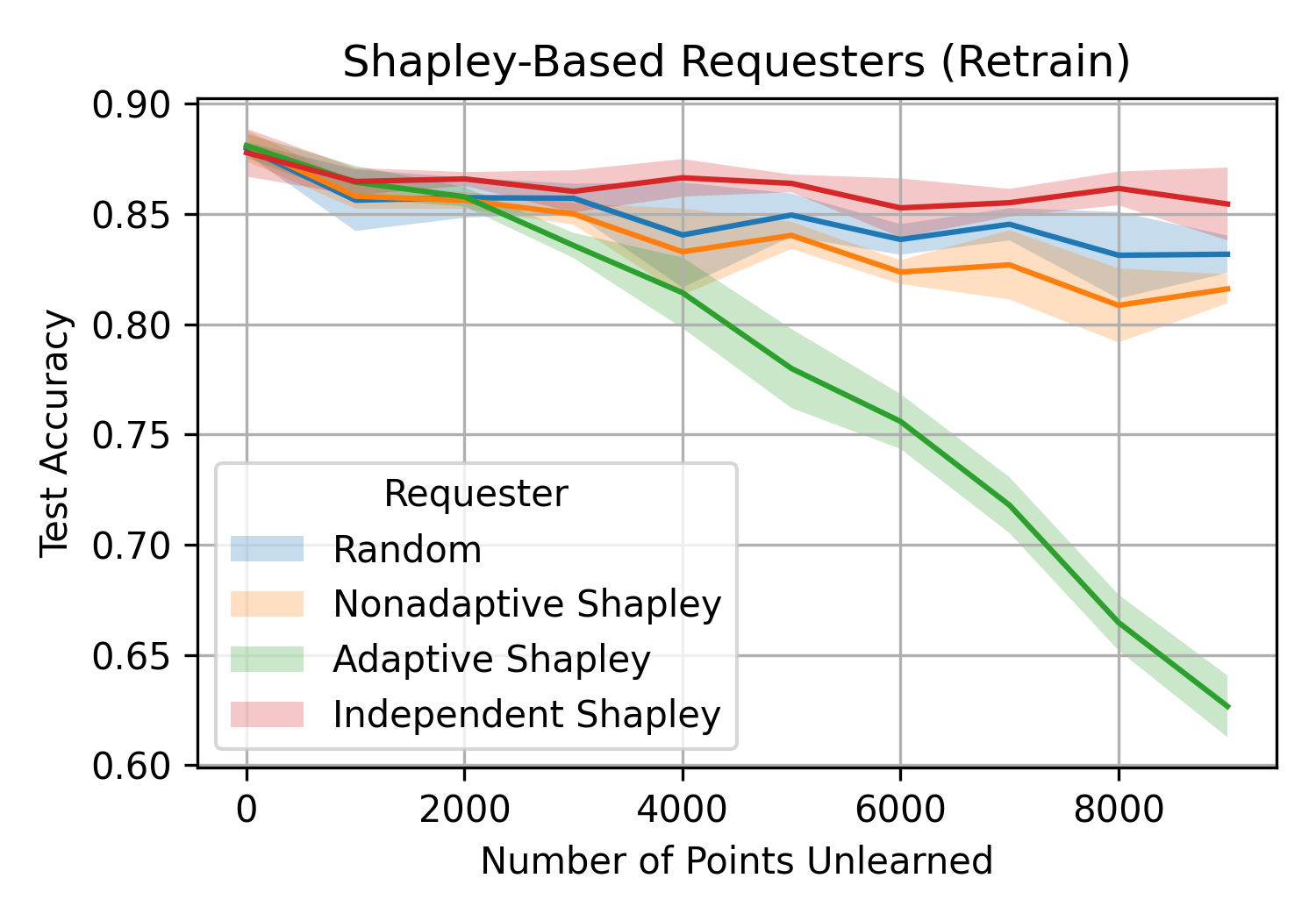}
    \caption[Shapley-Based Requester Unlearning Performance.]{\textbf{Shapley-Based Requester Unlearning Performance.} We plot the performance of variants of the Shapley requester (nonadaptive, adaptive, and individual) compared to the random requester.}
    \label{fig:shap}
\end{figure}
Additionally, we plot the results for the variants of the Shapley requester in Figure~\ref{fig:shap}. Unsurprisingly, the individual Shapley requester is less powerful than both the adaptive and nonadaptive variants of the (collective) Shapley requester. Oddly, the individual Shapley requester also has less utility than the random requester. This is likely due to the high-variance estimator of the distributional Shapley value that we utilize. Results from \citet{ghorbani2020distributional} suggest that the individual Shapley requester should in principle be able to outperform the random requester in at least some settings.

\subsubsection{Comparing Unlearning Algorithms}
\begin{figure}[ht]
    \centering
    \includegraphics[width=\linewidth]{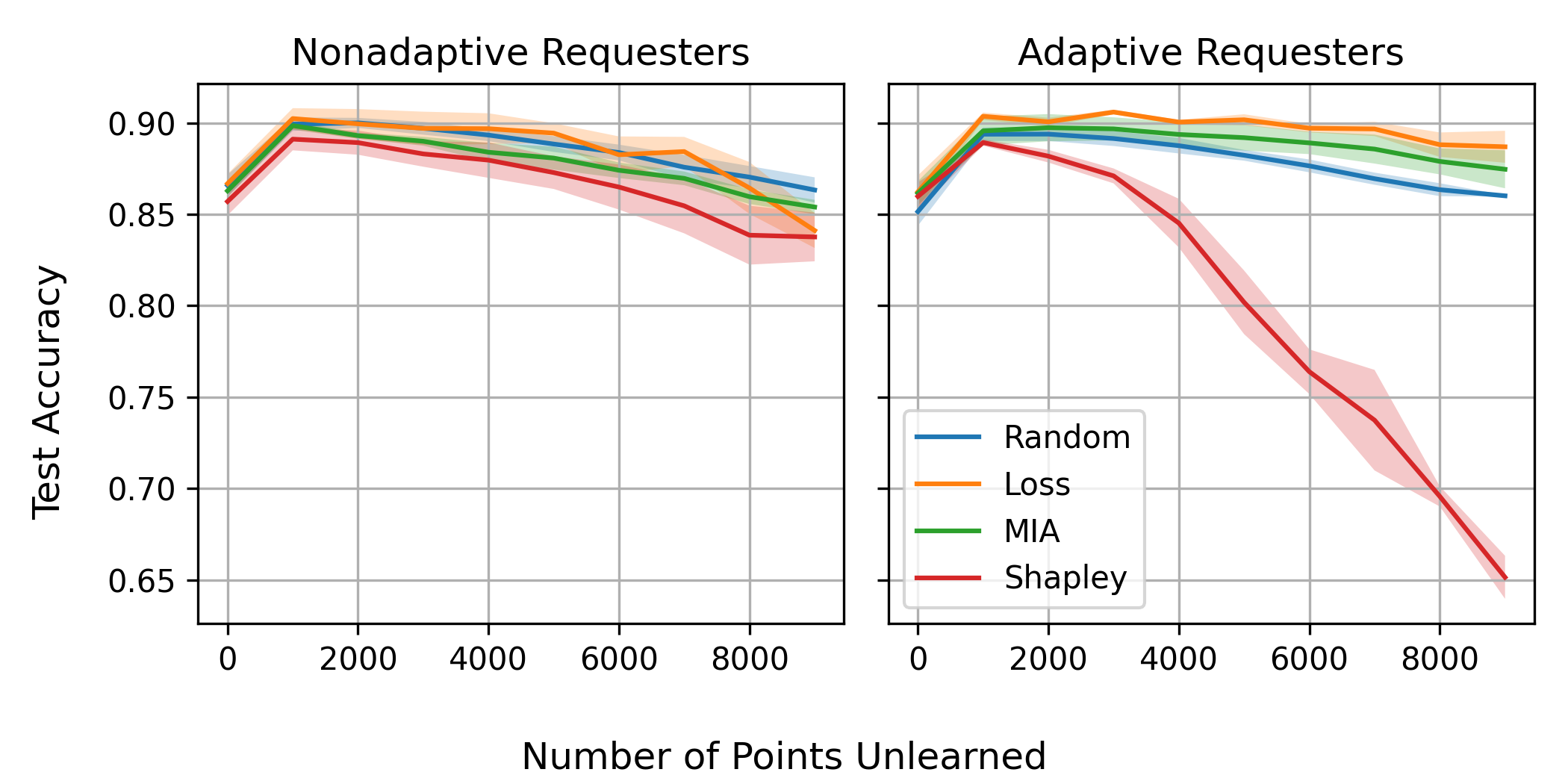}
    \caption[Relabel Unlearning Utility.]{\textbf{Relabel Unlearning Utility.} We plot the performance of nonadaptive and adaptive unlearning requesters compared to the random requester when the unlearning algorithm is Relabel.}
    \label{fig:relabel}
\end{figure}

\begin{figure}[ht]
    \centering
    \includegraphics[width=\linewidth]{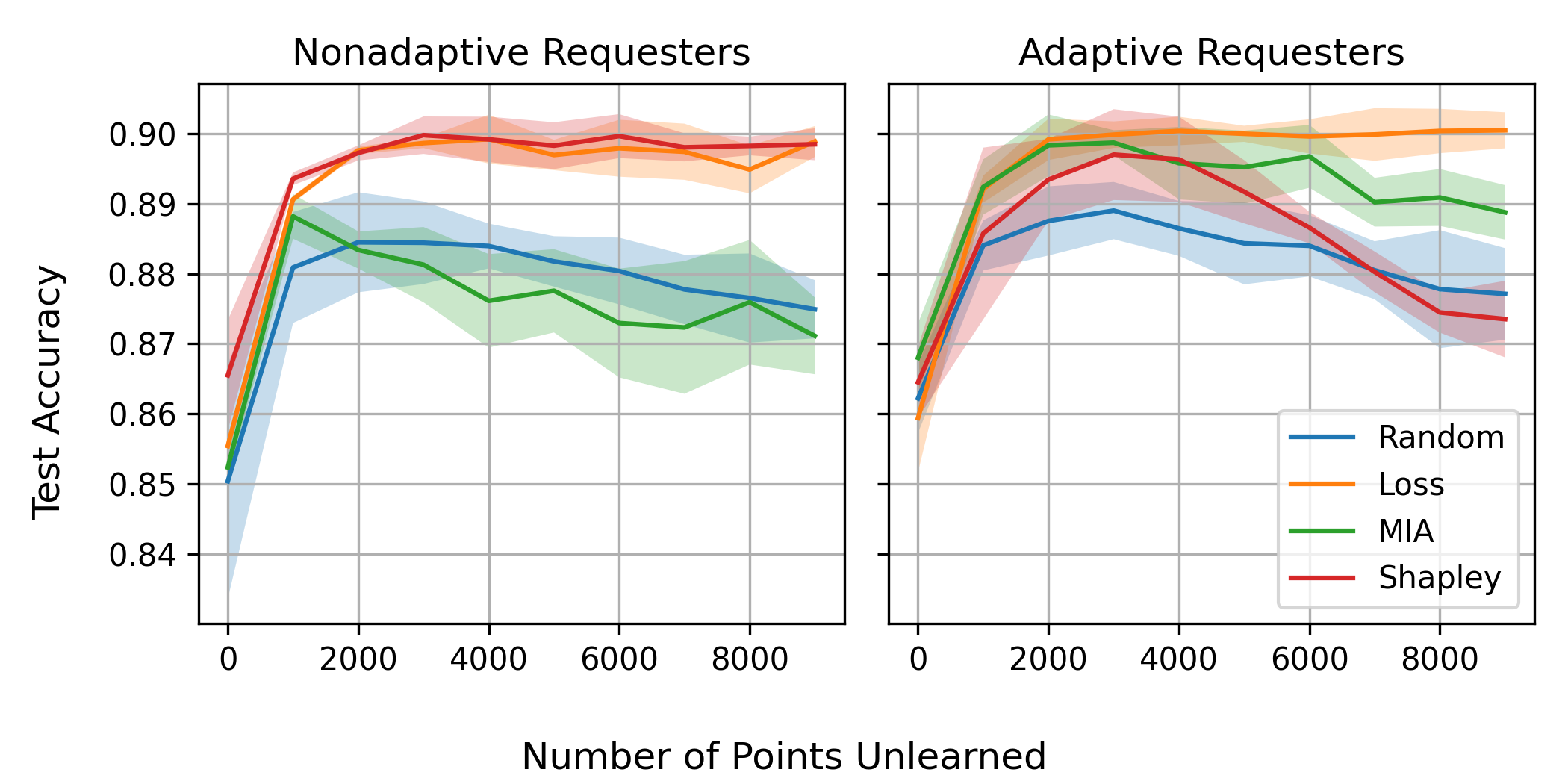}
    \caption[SCRUB Unlearning Utility.]{\textbf{SCRUB Unlearning Utility.} We plot the performance of nonadaptive and adaptive unlearning requesters compared to the random requester when the unlearning algorithm is SCRUB.}
    \label{fig:scrub}
\end{figure}

We recreate the plots in Figures \ref{fig:nonadaptive} and \ref{fig:adaptive} (which used Retrain as an unlearning method) using both Relabel (Figure \ref{fig:relabel}) and SCRUB (Figure \ref{fig:scrub}). We observe that for all requesters interacting with either of these approximate methods, the accuracy of the model increases after the first round of unlearning. We believe this is because both of these methods involve fine-tuning the weights of the original model, resulting in a model that has been trained for longer, and only after the first step have enough points been removed to for accuracy to start decreasing again.

Looking at the results for Relabel, patterns look broadly similar to those of Retrain. In particular, the nonadaptive Shapley requester shows slightly higher utility than the other requesters, and the adaptive Shapley requester displays significantly higher utility. Interestingly, the adaptive Loss and MIA requesters appear to have lower utility than the random requester, a pattern that is not clearly visible in their nonadaptive variants or when using Retrain.

SCRUB yields significantly different results. After the initial rise in accuracy, the accuracy of the unlearned models tends to decrease only marginally or not at all. In fact, the nonadaptive Shapley requester does not appear to reduce the accuracy of the model after 9,000 deletions. We do observe an gap between the adaptive and nonadaptive Shapley requester after $\sim$3,000 points are unlearned, but the effects are much more muted than with Retrain or Relabel. 

These plots illustrate how the utility of a requester is dependent on the unlearning method being used. Our results indicate that some approximate unlearning algorithms (like SCRUB) may be more resistant to the effects of requester adaptivity than others (like Relabel). More work is necessary to better understand this phenomenon.

\section{Limitations and Future Work}
There are inherent limitations in the scope of our theoretical results. These results only hold for the very simple learning task that we present; analyzing more complex settings may yield tighter bounds on the true adaptivity and collectivity gaps. We also define adaptivity and collectivity as binary attributes, while in reality they may be varying degrees of these properties that requests can display. Future work might focus on settings where small groups may act collectively, or where requests can only be dependent on recent model updates. Other limitations lie in our experimental setup. We focus only on one architecture and dataset, and our implementations of the variants of the Shapley requester rely on approximations that may obscure some of the true effects of unlearning. Our results should therefore be viewed as a proof of concept rather than exhaustive. Additional empirical work is necessary to better understand the impacts of different properties of requesters in realistic settings.

Another significant limitation of our framework is that we don't account for mitigation strategies that can be used by an unlearning algorithm to lessen the impacts of deletion requesters. For example, the algorithm could attempt to sample new data from the same distribution as the deleted data that can be used for fine-tuning the model after unlearning. It is not immediately clear from our results whether the effects of request distribution would be difficult to account for. One approach might take inspiration from concepts in performative prediction \citep{perdomo2020performative}. For example, one might attempt to train a model that induces distributions of requests from certain requesters that are close to random.

\section{Open Science}
Replication code is available at 
\url{https://github.com/chriscianfarani/ModelingDeletionRequests}.

\section{Ethical Considerations}
Much of this work studies deletion requesters acting in an adversarial manner. When viewed in this light, one might fear real harms being caused by coordinated deletion requests. These harms may take the form of increased algorithmic bias, unfair decisions, or abusive behavior from models, among others. However, we note that our theoretical framework was designed to be intentionally broad in order to capture a wide range of possible requester behaviors. Our focus on worst-case behaviors is intended to help us understand the limitations of different types of requesters. We also note that coordinated deletion requests may be used to remove harmful behaviors from models in a similar manner to the approaches proposed by \citet{hardt2023algorithmic}.

\bibliographystyle{IEEEtranN}
\bibliography{IEEEabrv,references}

\appendix 
\subsection{Existing Bounds on the Adaptivity Gap in Stochastic Knapsack}\label{app:prior_gaps}
\citet{gupta2011approximation} give an upper bound for the adaptivity gap for correlated stochastic knapsack, constructing a nonadaptive algorithm which is an 16-approximation of the optimal adaptive algorithm. 
Their formulation of the stochastic knapsack problem allows for \emph{cancellations}, which would correspond to giving up on the deletion of a subpopulation before it is fully removed and never returning to it. However, our formulation more broadly allows for \emph{preemptions}, where a requester might stop deleting a certain subpopulation but return to it later. \citet[Appendix B]{ma2018improvements} describe a setting in which preemption is necessary for the optimal solution in uncorrelated stochastic knapsack. However, they remark that the gap between algorithms that can preempt and algorithms that can only cancel appears to be bounded by a constant factor. If this is the case, then the algorithm given by \citet{gupta2011approximation} would also be a constant factor approximation of the optimal preempting policy for correlated stochastic knapsack and would serve as an upper bound for the individual adaptivity gap in our setting.

\citet{dean2008approximating} give a lower bound  of 1.25 for the adaptivity gap for the standard stochastic knapsack, but this result does not apply to our instantiation of the correlated stochastic knapsack problem.
\citet{dean2008approximating} arrive at this lower bound by formulating an instance of stochastic knapsack where the item weights are distributed according to carefully crafted distributions (which are not Poisson) and showing that the expected value of the optimal adaptive algorithm is at least 1.25 times greater than that of the optimal nonadaptive algorithm.
It does not follow from this result that the optimal adaptive algorithm has a higher expected value than the optimal nonadaptive algorithm in the setting where item weights are distributed as Poissons, as they are in our instantiation of the problem. \citet{barak2026stochastic} improve the lower bound on the adaptivity gap for stochastic knapsack to 1.37, but they also make use of carefully crafted non-Poisson item weight distributions.

\subsection{Proof of Claim~\ref{claim:opt-model}}\label{app:claim}
\begin{proof}
    In order for $\alg^*$ to be optimal, it must minimize the following objective:
    \[\inf_\alg \expect_{D \sim \D}\err(\alg(D)).\]
    Expanding this expression, we get
    \begin{align*}
        \inf_\alg \expect_{\substack{D \sim \D\\\theta \leftarrow \alg(D)}}\expect_{(x,y)\sim\D}\mathbbm{1}[\theta(x) \neq y]
    \end{align*}
    We can decompose the inner expectation into error on points that are in $D$ and error on points that are not in $D$.
    \ifsubmission
    \begin{align*}
        \expect_{\substack{D \sim \D\\\theta \leftarrow \alg(D)}}&\expect_{(x,y)\sim\D}\mathbbm{1}[\theta(x) \neq y]\\
        &=\expect_{\substack{D \sim \D\\\theta \leftarrow \alg(D)}}\expect_{(x,y)\sim\D}\Big[\mathbbm{1}[(x,y) \in D]\cdot\mathbbm{1}[\theta(x) \neq y]\\ 
        &\qquad+ \mathbbm{1}[(x,y) \notin D]\cdot\mathbbm{1}[\theta(x) \neq y]\Big]\\
        &=\expect_{\substack{D \sim \D\\\theta \leftarrow \alg(D)}}\Big[\expect_{(x,y)\sim\D}\mathbbm{1}[\theta(x) \neq y \land (x,y) \in D] \\
        &\qquad+ \expect_{(x,y)\sim\D}\mathbbm{1}[\theta(x) \neq y\land(x,y) \notin D]\Big]\\
        &=\expect_{\substack{D \sim \D\\\theta \leftarrow \alg(D)}}\Big[\Pr_{(x,y)\sim\D}[\theta(x) \neq y \land (x,y) \in D]\\ 
        &\qquad+ \Pr_{(x,y)\sim\D}[\theta(x) \neq y\land(x,y) \notin D]\Big].
    \end{align*}
    \else
    \begin{align*}
        \expect_{\substack{D \sim \D\\\theta \leftarrow \alg(D)}}&\expect_{(x,y)\sim\D}\mathbbm{1}[\theta(x) \neq y]\\
        &=\expect_{\substack{D \sim \D\\\theta \leftarrow \alg(D)}}\expect_{(x,y)\sim\D}\Big[\mathbbm{1}[(x,y) \in D]\cdot\mathbbm{1}[\theta(x) \neq y] + \mathbbm{1}[(x,y) \notin D]\cdot\mathbbm{1}[\theta(x) \neq y]\Big]\\
        &=\expect_{\substack{D \sim \D\\\theta \leftarrow \alg(D)}}\expect_{(x,y)\sim\D}\mathbbm{1}[\theta(x) \neq y \land (x,y) \in D] + \expect_{(x,y)\sim\D}\mathbbm{1}[\theta(x) \neq y\land(x,y) \notin D]\\
        &=\expect_{\substack{D \sim \D\\\theta \leftarrow \alg(D)}}\Pr_{(x,y)\sim\D}[\theta(x) \neq y \land (x,y) \in D] + \Pr_{(x,y)\sim\D}[\theta(x) \neq y\land(x,y) \notin D].
    \end{align*}
    \fi
    The left inner probability in the last line is minimized by $\alg^*$. Letting $\theta^* = \alg^*(D)$ $\theta^*(x) = y$ for all $(x,y) \in D$, so this probability will be 0. By Bayes' rule, the right inner probability is equivalent to
    \[\Pr_{(x,y)\sim\D}[\theta(x) \neq y\mid(x,y) \notin D]\cdot\Pr[(x,y) \notin D].\]
    $\Pr[(x,y) \notin D]$ is independent of $\alg$, since $\alg$ and $\D$ are fixed before the dataset is sampled. 
    Let $f$ be the labeling function used to label the points in $D$. 
    Now, consider a point $x \in \X$ such that $(x,f(x)) \notin D$. 
    Since the labeling function for $D$ was sampled from the uniform distribution over all binary labeling functions, its label is independent of the label of any subpopulation in $D$. This implies that $\theta = \alg(D)$ is also independent of $f(x)$ for all learning algorithms $\alg$ that are independent of $f$ conditioned on the training dataset $D$. Therefore, the right probability above is a constant which is minimized by any learning algorithm, so $\alg^*$ minimizes the objective.
\end{proof}

\subsection{Randomized Individual Requesters}
\label{sec:ind_requesters}
We note that while an individual requester does not a priori know the exact multiplicity of each subpopulation in the dataset, it is able to sequentially request the deletion of complete subpopulations. 
In fact, since we require scoring functions be deterministic and disallow ties between subpopulations, nonadaptive individual requesters are only capable of removing whole subpopulations sequentially. On the other hand, adaptive individual requesters, being able to act on incremental model updates, are able to \emph{preempt} the deletion of a given subpopulation before it is fully removed.
Similarly, an individual requester can determine when a complete subpopulation has been removed by observing the point at which the accuracy of the model decreases. (This is where the choice described in footnote~\ref{footnote:guessing} comes into play.)

Ideally, we would be able to model requesters characterized by randomized scoring functions. 
Under the following conjecture, which states that there exists an optimal randomized individual requester of a simple form, the problem of finding an optimal individual requester would also reduce to correlated stochastic knapsack as above.
\begin{conjecture}
    An optimal randomized individual deletion requester (adaptive or nonadaptive) will never \emph{preempt}
    a subpopulation. That is to say, such a requester will never begin removing 
    samples from a certain population and switch to removing samples from a 
    different subpopulation before the first subpopulation was completely 
    removed.
\end{conjecture}
In fact, we prove a weaker version of this conjecture in Theorem~\ref{thm:adaptive_collective} which applies to a setting with only two subpopulations. We hypothesize that this conjecture is true for both adaptive and nonadaptive requesters in all settings where $n_i$ is sampled from a Poisson distribution, as is the case here. When a copy of a subpopulation is deleted, the expected number of copies off that subpopulation that remain in the training data will strictly decrease, while the utility gained by removing that subpopulation in its entirety will remain the same. In that sense, the \emph{expected utility per deletion} (the ratio of utility gained by removing a complete subpopulation to number of remaining elements of that subpopulation) will increase for a given subpopulation after one of its elements is deleted and remain the same for all other subpopulations. It would make sense that a requester would not preempt a subpopulation whose expected utility per deletion is increasing in this way. However, this analysis overlooks the requester's total budget of $k$ deletions, which may induce non-monotonicity in expected utility per deletion for a given subpopulation. In our attempts, we have failed to either prove this conjecture or find a counterexample. \citet{ma2018improvements} demonstrates that there are instances of the stochastic knapsack problem where preemption is necessary to obtain the optimal value. However, like the prior lower bounds on the adaptivity gap of stochastic knapsack, this result requires that item sizes be sampled from carefully crafted non-Poisson distributions.

\subsection{Adaptivity Gap Visualizations (Individual)}\label{app:viz}
Figures \ref{fig:heatmap} and \ref{fig:square_heatmap} visualize our grid search to find the maximal adaptivity gap in a three subpopulation setting. Figure \ref{fig:heatmap} shows a search over values of $p_1$, $p_2$, and $p_3$, representing subpopulation frequencies. Figure \ref{fig:square_heatmap} fixes $p_1 = p_2 = \epsilon$ and $p_3 = 1 - 2\epsilon$, and searches over values of $\epsilon$ and $N$.
\begin{figure}[ht]
    \centering
    \includegraphics[width=\linewidth]{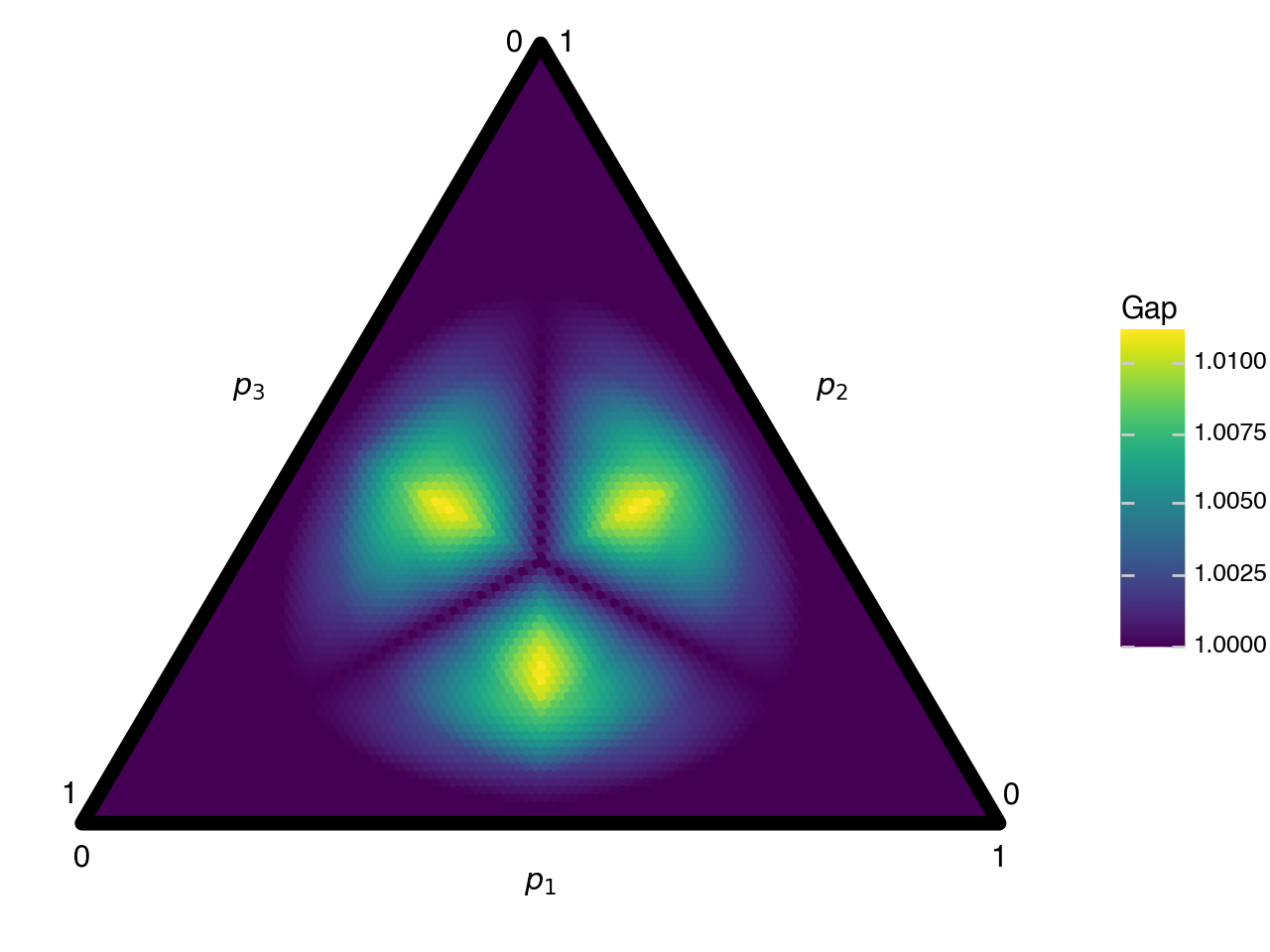}
    \caption[Thresholded Adaptivity Gaps.]{\textbf{Thresholded Adaptivity Gaps.} A ternary plot of the thresholded adaptivity gap for three subpopulations plotted over values of $(p_1,p_2,p_3)$. In this plot, $N$ is fixed to 10 and $k$ is fixed to 8.}
    \label{fig:heatmap}
\end{figure}
\begin{figure}[ht]
    \centering
    \includegraphics[width=\linewidth]{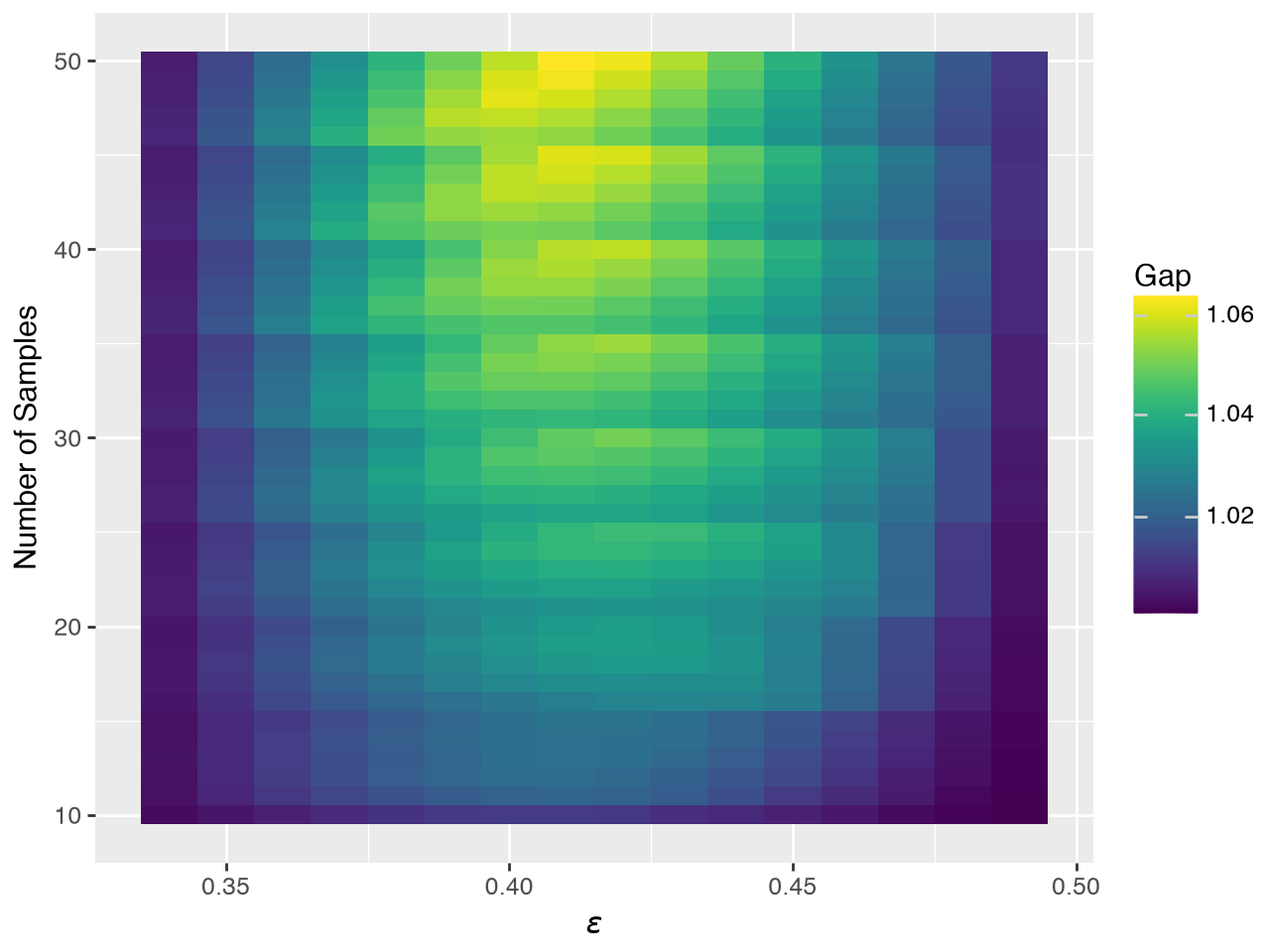}
    \caption[Thresholded Adaptivity Gaps in a Single Parameter Setting.]{\textbf{Thresholded Adaptivity Gaps in a Single Parameter Setting.} Plotting the thresholded adaptivity gap over values of $\varepsilon$ and expected number of samples $N$, such that $p_1 = p_2 = \varepsilon$ and $p_3 = 1 -2\varepsilon$. In this plot, we fix $k$ to $\lfloor 0.8N \rfloor$.}
    \label{fig:square_heatmap}
\end{figure}

\subsection{The Individual Shapley Requester and the Distributional Shapley Value}\label{app:ind_shap}
As we have defined it, the individual Shapley requester is related to the formulation of the Distributional Shapley Value introduced by \citet{ghorbani2020distributional}:
\begin{definition}[Distributional Shapley Value \citep{ghorbani2020distributional}]
    Given a data distribution $\D$ supported on $\X \times \Y$, a utility function $\fn: (\X \times \Y)^* \rightarrow \reals$, and some $m \in \mathbb{N}$, the distributional Shapley value for a point $(x,y) \in \X \times \Y$ is the expected data Shapley value over data sets of size $m$ containing $i$:
    \[\nu((x,y);\fn, \D, m) = \expect_{B \sim \D^{m-1}} \phi_{(x,y)}(B \cup \{(x,y)\}, \fn).\]
\end{definition}
Since the individual Shapley requester only uses a single sample from the distribution to determine its ordering, it is implicitly computing an unbiased (but high variance) estimator of the distributional Shapley value. We made this design decision for a number of reasons. For one, the ordering computed by the individual Shapley requester needs to be independent of the training dataset $D$. In our experiments we work with the CIFAR-10 dataset, which is of modest size (50,000 training examples). In order to sample datasets from the same distribution, the requester must train on a subset of CIFAR-10 and compute the ordering using subsets of the same size. If we train on a large subset of CIFAR-10, there will be significant overlap between the training sets and the subsets used for Shapley computation, increasing the dependence between the distributional Shapley computation and the true Shapley value with regards to $D$. If we train on a small subset of CIFAR-10, there will be less overlap but the performance of the trained classifiers will be poor. Training on half of the training set and computing the deletion ordering with respect to the held out set allows us to ensure that the requester is truly individual.

\subsection{Proof of Claim~\ref{claim:shap_subopt}}\label{app:shap_subopt}
\begin{proof}
    We will demonstrate a realized dataset in which the Shapley requester does not behave optimally. Let $\X = \{a,b\}$, and let $p_a = \frac{1}{3}$ and $p_b = \frac{2}{3}$. Suppose that a dataset $D^\mathsf{ref}$ is sampled such that it contains 1 item from subpopulation $a$ and 2 items from subpopulation $b$. Furthermore, let $k=1$.

    In this setting, the test accuracy of a model trained on $S$ using learning algorithm $\alg^*$ is 
    \begin{align*}
        u(S) &= \sum_{i \in \X} p_i \cdot \mathbbm{1}[i \in S] + \frac{p_i}{2} \cdot \mathbbm{1}[i \notin S].
    \end{align*}
    It is clear from this expression that $u(S \cup \{(i,f(i))\}) > u(S)$ if and only if $(i,f(i)) \notin S$. In fact, if $(i,f(i)) \notin S$, then $u(S \cup \{(i,f(i))\}) - u(S) = \frac{p_i}{2}$. We can therefore simplify the Shapley value in this scenario:
    \begin{align*}
        \phi_i &= \frac{1}{n}\cdot\sum_{S \subseteq D^\mathsf{ref}_{-i}}\frac{u(S \cup\{(x_i,y_i)\}) - u(S)}{\binom{n - 1}{|S|}}\\
        &= \frac{1}{n}\cdot\sum_{S\subseteq D^\mathsf{ref}_{-i}}\frac{p_{x_i}/2\cdot\mathbbm{1}[(x_i,y_i) \notin S]}{\binom{n-1}{|S|}}.
    \end{align*}

    Let $j_a$ denote the index in $D^\mathsf{ref}$ of the item from subpopulation $a$. None of the subsets of $D^\mathsf{ref}_{-j_a}$ contain subpopulation $a$. Therefore, the Shapley value for this item is 
     \begin{align*}
        \phi_{j_a} &= \frac{1}{n}\cdot\sum_{S\subseteq D^\mathsf{ref}_{-j_a}}\frac{p_a/2\cdot\mathbbm{1}[(x_{j_a},y_{j_a}) \notin S]}{\binom{n-1}{|S|}}\\
        &= \frac{1}{3} \cdot \sum_{S\subseteq D^\mathsf{ref}_{-j_a}} \frac{p_a/2}{\binom{2}{|S|}}\\
        &= \frac{p_a}{6}\left(\frac{1}{\binom{2}{|\{\}|}} + \frac{1}{\binom{2}{|\{b\}|}} + \frac{1}{\binom{2}{|\{b\}|}} + \frac{1}{\binom{2}{|\{b,b\}|}}\right)\\
        &= \frac{p_a}{6}\left(\frac{1}{1} + \frac{1}{2} + \frac{1}{2} + \frac{1}{1}\right)\\
        &= \frac{p_a}{2} = \frac{1}{6}.
    \end{align*}

    Let $j_b$ be the index in $D^\mathsf{ref}$ of one of the items representing subpopulation $b$. The subsets of $D^\mathsf{ref}_{-j_b}$ that don't contain subpopulation $b$ are $\{\}$ and $\{a\}$. Therefore, the Shapley value for an element of subpopulation $b$ is 
     \begin{align*}
        \phi_{j_b} &= \frac{1}{n}\cdot\sum_{S\subseteq D^\mathsf{ref}_{-j_b }}\frac{p_b/2\cdot\mathbbm{1}[(x_{j_b},y_{j_b}) \notin S]}{\binom{n-1}{|S|}}\\
        &= \frac{1}{3} \cdot \sum_{S \in \{\{\},\{a\}\}} \frac{p_b/2}{\binom{2}{|S|}}\\
        &= \frac{p_b}{6} \cdot \left(\frac{1}{\binom{2}{|\{\}|}} +\frac{1}{\binom{2}{|\{a\}|}}\right)\\
        &= \frac{p_b}{6} \cdot \left(\frac{1}{1} +\frac{1}{2}\right)\\
        &= \frac{p_b}{4} = \frac{1}{6}.
    \end{align*}

    Therefore, $\phi_{j_a} = \phi_{j_b}$. The Shapley requester, which breaks ties randomly, will request the deletion of a random sample from $D^\mathsf{ref}$ and expect to achieve a utility of
    \ifsubmission
    \begin{align*}
        \expect_\shap\util_{\shap,\err}(\theta^0,D^\mathsf{ref},1) &= \expect_\shap\err(\theta^1) - \err(\theta^0)\\
        &= \expect_\shap\err(\theta^1)\\
        &= \Pr[u^1 = a] \cdot \frac{p_a}{2}\\
        &= \frac{p_a}{6} = \frac{1}{18}.
    \end{align*}
    \else
    \begin{align*}
        \expect_\shap\util_{\shap,\err}(\theta^0,D^\mathsf{ref},1) &= \expect_\shap\err(\theta^1) - \err(\theta^0)\\
        &= \expect_\shap\err(\theta^1)\\
        &= \Pr[u^1 = a] \cdot \frac{p_a}{2} + \Pr[u^1 = b] \cdot 0\\
        &= \frac{p_a}{6} = \frac{1}{18}.
    \end{align*}
    \fi
    Now, let $\shap'$ be a requester that behaves identically to $\shap$ unless $D=D^\mathsf{ref}$, in which case it will request deletions in the optimal order (subpopulation $a$, then subpopulation $b$).
    On dataset $D^\mathsf{ref}$, $\shap'$ would deterministically achieve a utility of 
    \begin{align*}
        \util_{\shap',\err}(\theta^0,D^\mathsf{ref},1) &= \err(\theta_1) - \err(\theta^0)\\
        &= \frac{p_a}{2} = \frac{1}{6}.
    \end{align*}
    This implies that 
    \ifsubmission
    \begin{align*}
    \expect_{D \sim \D}&[\util_{\shap,\err}(\theta,D,k) \mid D \neq D^\mathsf{ref}]\\ 
    &= \expect_{D \sim \D}[\util_{\shap',\err}(\theta,D,k) \mid D \neq D^\mathsf{ref}]
    \end{align*}
    and 
    \begin{align*}
    \expect_{D \sim \D}&[\util_{\shap,\err}(\theta,D,k) \mid D = D^\mathsf{ref}]\\ 
    &< \expect_{D \sim \D}[\util_{\shap',\err}(\theta,D,k) \mid D = D^\mathsf{ref}],
    \end{align*}
    \else
    \[\expect_{D \sim \D}[\util_{\shap,\err}(\theta,D,k) \mid D \neq D^\mathsf{ref}] = \expect_{D \sim \D}[\util_{\shap',\err}(\theta,D,k) \mid D \neq D^\mathsf{ref}]\]
    and 
    \[\expect_{D \sim \D}[\util_{\shap,\err}(\theta,D,k) \mid D = D^\mathsf{ref}] < \expect_{D \sim \D}[\util_{\shap',\err}(\theta,D,k) \mid D = D^\mathsf{ref}],\]
    \fi
    Therefore, $\overline{\util}_{\shap,\err}(\D,k) < \overline{\util}_{\shap',\err}(\D,k)$,
    so $\shap$ is subpotimal.
\end{proof}

\end{document}